\documentclass{article}

\usepackage[preprint]{neurips_2026}

\usepackage[utf8]{inputenc} 
\usepackage[T1]{fontenc}    
\usepackage{hyperref}       
\usepackage{url}            
\usepackage{booktabs}       
\usepackage{amsfonts}       
\usepackage{nicefrac}       
\usepackage{microtype}      
\usepackage{xcolor}         

\usepackage{amsmath}
\usepackage{amssymb}
\usepackage{mathtools}
\usepackage{amsthm}
\usepackage{graphicx}
\usepackage{multirow}

\usepackage{algorithm}
\usepackage{algorithmicx}
\usepackage{algpseudocode}
\algnewcommand\algorithmicto{\textbf{to}}
\algrenewcommand\algorithmicrequire{\textbf{Input:}}
\algrenewcommand\algorithmicensure{\textbf{Output:}}
\algdef{SE}[SUBALG]{Indent}{EndIndent}{}{\algorithmicend\ }%
\algtext*{Indent}
\algtext*{EndIndent}

\usepackage{tikz}
\usetikzlibrary{arrows.meta,positioning,fit,shapes.geometric,calc}

\theoremstyle{plain}
\newtheorem{theorem}{Theorem}[section]
\newtheorem{proposition}[theorem]{Proposition}
\newtheorem{lemma}[theorem]{Lemma}

\theoremstyle{definition}

\newtheorem{condition}[theorem]{Condition}
\theoremstyle{remark}

\title{Reducing Credit-Assignment Variance Through Counterfactual Reasoning Paths}

\author{%
  Fei Ding \\
  Alibaba Group \\
  \texttt{dignfei@gmail.com} \\
  \And
  Yongkang Zhang \\
  Alibaba Group \\
  \And
  Youwei Wang \\
  Tsinghua University \\
  \And
  Zijian Zeng \\
  Tsinghua University
}

\begin{document}

\maketitle

\begin{abstract}
    Reinforcement learning for multi-step reasoning with large language models (LLMs) typically relies on sparse terminal rewards, which creates a poorly conditioned credit-assignment problem: the final feedback is propagated uniformly across all intermediate decisions. This leads to high gradient variance, unstable training, and many ineffective updates, ultimately limiting sustained model improvement. We propose a counterfactual-comparison framework for credit assignment. For each input, the framework samples multiple reasoning trajectories and treats their differences as implicit approximations to alternative decisions. This yields an implicit process-level advantage estimator that converts sparse terminal rewards into step-sensitive learning signals. Building on this framework, we introduce Implicit Behavior Policy Optimization (IBPO), which substantially improves training stability and the performance ceiling on mathematical and code-reasoning benchmarks. Our results point to a promising direction for unlocking the reasoning potential of LLMs.
\end{abstract}

\section{Introduction}

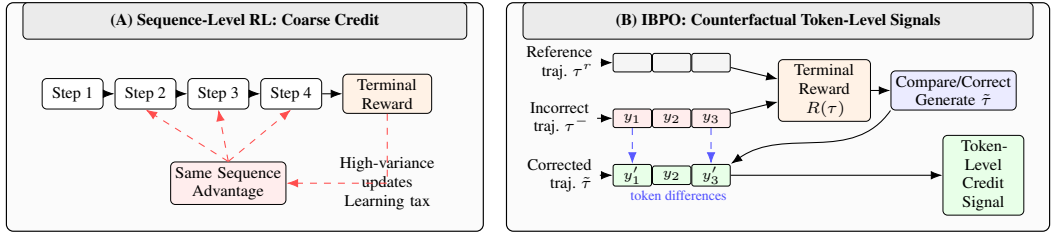
\begin{figure}[H]
    \centering
    \resizebox{\linewidth}{!}{
    \begin{tikzpicture}[
        font=\scriptsize,
        >=Latex,
        panel/.style={
            draw,
            rounded corners=3pt,
            line width=0.5pt,
            fill=gray!3,
            inner sep=4pt
        },
        title/.style={
            draw,
            rounded corners=2pt,
            fill=gray!16,
            align=center,
            minimum height=0.55cm,
            inner sep=3pt,
            font=\scriptsize\bfseries
        },
        flowstep/.style={
            draw,
            rounded corners=2pt,
            align=center,
            minimum height=0.46cm,
            minimum width=0.92cm,
            inner sep=2pt,
            fill=white
        },
        signal/.style={
            draw,
            rounded corners=2pt,
            align=center,
            minimum height=0.58cm,
            text width=1.35cm,
            inner sep=2pt,
            fill=orange!10
        },
        update/.style={
            draw,
            rounded corners=2pt,
            align=center,
            minimum height=0.58cm,
            text width=1.55cm,
            inner sep=2pt,
            fill=green!9
        },
        token/.style={
            draw,
            rounded corners=1pt,
            align=center,
            minimum height=0.28cm,
            minimum width=0.58cm,
            inner sep=1pt,
            fill=blue!5,
            font=\tiny
        },
        arr/.style={->, line width=0.45pt},
        softarr/.style={->, line width=0.45pt, blue!65},
        badarr/.style={->, dashed, line width=0.45pt, red!70}
    ]
        \node[panel, minimum width=7.2cm, minimum height=3.45cm] (leftpanel) at (3.6, 0) {};
        \node[panel, minimum width=8.2cm, minimum height=3.45cm] (rightpanel) at (11.65, 0) {};

        \node[title, minimum width=6.7cm] at (3.6, 1.45) {(A) Sequence-Level RL: Coarse Credit};
        \node[flowstep] (l1) at (1.0, 0.35) {Step 1};
        \node[flowstep] (l2) at (2.1, 0.35) {Step 2};
        \node[flowstep] (l3) at (3.2, 0.35) {Step 3};
        \node[flowstep] (l4) at (4.3, 0.35) {Step 4};
        \node[signal, text width=1.2cm] (lr) at (5.75, 0.35) {Terminal\\Reward};
        \node[signal, text width=1.6cm, fill=red!7] (lc) at (3.35, -1.0) {Same Sequence\\Advantage};
        \node[align=center, text width=2.1cm] at (5.75, -1.0) {High-variance\\updates\\Learning tax};
        \draw[arr] (l1) -- (l2);
        \draw[arr] (l2) -- (l3);
        \draw[arr] (l3) -- (l4);
        \draw[arr] (l4) -- (lr);
        \draw[badarr] (lr.south) |- (lc.east);
        \draw[badarr] (lc.north) -- (l2.south);
        \draw[badarr] (lc.north) -- (l3.south);
        \draw[badarr] (lc.north) -- (l4.south);

        \node[title, minimum width=7.7cm] at (11.65, 1.45) {(B) IBPO: Counterfactual Token-Level Signals};
        \node[align=right] (rlab) at (8.35, 0.82) {Reference\\traj. $\tau^r$};
        \node[token, fill=gray!8] (r1) at (9.45, 0.82) {};
        \node[token, fill=gray!8, right=0pt of r1] (r2) {};
        \node[token, fill=gray!8, right=0pt of r2] (r3) {};
        \node[align=right] (elab) at (8.35, -0.02) {Incorrect\\traj. $\tau^-$};
        \node[token, fill=red!7] (e1) at (9.45, -0.02) {$y_1$};
        \node[token, fill=red!7, right=0pt of e1] (e2) {$y_2$};
        \node[token, fill=red!7, right=0pt of e2] (e3) {$y_3$};
        \node[signal, text width=1.25cm] (ver) at (12.35, 0.4) {Terminal\\Reward\\$R(\tau)$};
        \node[signal, text width=1.75cm, fill=blue!8] (corr) at (14.3, 0.4) {Compare/Correct\\Generate $\tilde{\tau}$};
        \node[align=right] (clab) at (8.35, -0.88) {Corrected\\traj. $\tilde{\tau}$};
        \node[token, fill=green!9] (c1) at (9.45, -0.88) {$y'_1$};
        \node[token, fill=green!9, right=0pt of c1] (c2) {$y_2$};
        \node[token, fill=green!9, right=0pt of c2] (c3) {$y'_3$};
        \node[update, text width=1.08cm] (out) at (14.75, -0.88) {Token-Level\\Credit\\Signal};

        \draw[arr] (elab) -- (e1);
        \draw[arr] (rlab) -- (r1);
        \draw[arr] (clab) -- (c1);
        \draw[arr] (e3) -- (ver);
        \draw[arr] (r3) -- (ver);
        \draw[arr] (ver) -- (corr);
        \draw[arr] (corr.south west) to[out=-125,in=35] (c3.north east);
        \draw[arr] (c3.east) -- (out.west);
        \draw[softarr, dashed] (e1.south) to[out=-90,in=90] (c1.north);
        \draw[softarr, dashed] (e3.south) to[out=-90,in=90] (c3.north);
        \node[font=\tiny, blue!70] at (10.15, -1.18) {token differences};
    \end{tikzpicture}
    }
    \caption{Overview of IBPO. Standard sequence-level RL propagates sparse terminal rewards coarsely to the whole trajectory, which can produce unstable updates. IBPO first generates an incorrect trajectory and a reference trajectory and obtains terminal feedback from a verifier; it then generates a corrected third trajectory through comparison/correction. The training signal is induced by the terminal reward and the token-level differences between the incorrect and corrected trajectories, converting sparse terminal feedback into token-level credit signals.}
    \label{fig:ibpo_overview}
\end{figure}

Recent advances in large language models (LLMs) have produced substantial progress on complex multi-step reasoning tasks, especially after reinforcement-learning (RL) fine-tuning. RL has become a central paradigm for scaling LLM capabilities, enabling models to solve increasingly difficult problems, such as competition-level mathematics and program synthesis, through deeper and longer reasoning chains.

However, scaling RL for reasoning requires training to remain stable and sample-efficient under growing compute budgets. Despite this need, mainstream RL methods, such as Group Relative Policy Optimization (GRPO)~\cite{shao2024deepseekmath}, still optimize policies with sequence-level or trajectory-level rewards. This creates a fundamental mismatch between the learning signal and the inherently stepwise nature of reasoning.

In multi-step reasoning, correctness depends on a sequence of intermediate decisions. Sequence-level supervision, however, rewards an entire trajectory only according to the final answer: a trajectory with flawed reasoning can still receive a positive reward if it happens to end with the correct answer, whereas a largely correct trajectory with a single local mistake may be discarded wholesale. This paper focuses primarily on the latter class of failed but recoverable trajectories. They contain useful intermediate reasoning fragments, yet under sequence-level rewards they are treated entirely as negative samples because their final answers are wrong. Such coarse feedback weakens the model's ability to distinguish early from late mistakes, disrupts credit assignment, destabilizes learning, and limits exploration of alternative reasoning paths. The problem is especially pronounced in long-horizon or difficult tasks. Moreover, even a single local error may require extensive sampling and many updates before it is corrected statistically, creating a substantial efficiency bottleneck often referred to as a \emph{learning tax}.

In this work, we propose a counterfactual learning approach to credit assignment under sparse terminal rewards. Even without step-level supervision, differences among reasoning trajectories sampled for the same input naturally contain process-level information. Disagreements between these trajectories implicitly reveal how alternative intermediate decisions may lead to different outcomes. By systematically comparing these counterfactual paths and aligning their differences with final outcomes, we construct informative learning signals that are more sensitive to intermediate decisions.

Building on this insight, we introduce \emph{Implicit Behavior Policy Optimization (IBPO)}, a framework for process-level credit assignment induced by counterfactual trajectory comparison. IBPO defines a general multi-trajectory comparison operator and uses it to construct an implicit advantage estimator. The estimator reweights terminal rewards according to trajectory-level differences, improving the within-group structure of the pre-normalization advantage signal and amplifying learning signals at frequent decision errors. IBPO does not require step-level annotations, process reward models, or an additional value network. It can be integrated seamlessly with existing sequence-level RL optimizers while improving convergence stability and sample efficiency.

\paragraph{Contributions.} Our main contributions are as follows:

\begin{itemize}
    \item \textbf{Counterfactual modeling of credit assignment.} We introduce a counterfactual-learning perspective on credit assignment in LLM reinforcement learning, treating multiple reasoning trajectories for the same input as approximations to alternative decisions. We show that inconsistencies among these trajectories contain key information for process-level learning, even in the absence of step-level rewards.

    \item \textbf{Implicit process-level advantages and the IBPO framework.} We formalize a general multi-trajectory comparison operator and use it to construct an implicit process-level advantage estimator, yielding the IBPO framework.

    \item \textbf{Analysis of variance reduction and positive transfer.} We provide a conditional variance analysis showing that, when the counterfactual comparison signal complements terminal-reward noise and the shaping weight is moderate, IBPO can reduce the dispersion of the pre-normalization, within-group centered advantage terms. Because GSPO further applies within-group standard-deviation normalization, the full policy-gradient variance also depends on the policy-gradient terms themselves. We therefore analyze how this mechanism improves the structure of update signals and mitigates the learning tax.

    \item \textbf{Mechanism-driven empirical validation.} We evaluate IBPO on multiple mathematical and code-reasoning benchmarks. Under matched training compute, IBPO consistently improves convergence, sample efficiency, and early-error correction, outperforming strong baselines.
\end{itemize}

\section{Related Work}

\textbf{Group Relative Policy Optimization (GRPO).} GRPO~\cite{shao2024deepseekmath} is a recent reinforcement-learning algorithm for fine-tuning LLMs on reasoning tasks and has achieved strong results in systems such as DeepSeek-R1~\cite{guo2025deepseek}. GRPO estimates group-relative advantages from within-group sampling, replacing the explicit value modeling used in PPO and enabling faster, more efficient training. However, GRPO suffers from entropy collapse, reward collapse, and unstable convergence~\cite{yu2025dapo}. These issues largely stem from its reliance on the assumption that the \emph{final reward is sufficient to characterize the reasoning trajectory}. This assumption often fails in long-horizon reasoning, where success depends on a sequence of interdependent steps, leading to ill-posed credit assignment and inflated gradient variance. GSPO~\cite{zheng2025group} improves on GRPO by computing the importance ratio at the sequence level.

\textbf{Self-correction strategies.} Self-correction has emerged as a promising direction for strengthening reasoning. For example, selective reflection tuning~\cite{li-etal-2024-selective} enables models to reflectively evaluate multiple candidate responses and then fine-tune on the best response with supervised learning.

\textbf{Reward modeling.} Reward models are crucial for robust System-2 reasoning, but they are difficult to build. Recent directions include LLM-as-a-Judge frameworks~\cite{llmjudge,rstar}, outcome reward models~\cite{qwen2025qwen25technicalreport,yu-etal-2024-ovm}, and process reward models (PRMs) that provide step-level feedback for complex tasks~\cite{lightman2024lets,luo2024improve,mathshepherd}. PRMs, however, face important limitations: high annotation cost, weak generalization, and noisy signals from automated methods such as Monte Carlo sampling or MCTS~\cite{kang2024mindstar,wang2024qimprovingmultistepreasoning}. Human-annotated datasets such as PRM800k~\cite{lightman2024lets} are difficult to scale, and existing automatic annotation methods often produce noisy or inconsistent reward scores. In contrast, IBPO bypasses the need for fine-grained annotation through implicit comparison while still providing effective process-level supervision. Unlike prior methods, our approach does not assume that the reward decomposes into stepwise reward signals.

\textbf{Rubric rewards and relational counterfactual shaping.} Recent work has also explored using explicit rubrics as reward functions to provide denser feedback for tasks that lack easily verifiable answers~\cite{gunjal2026rubrics}. These methods are typically \emph{criterion-based single-output scoring}: given one candidate response, an external rubric or judge scores its quality according to predefined criteria. IBPO focuses on a complementary source of signal: \emph{relational counterfactual differences among multiple on-policy trajectories for the same input}. In other words, rubric methods densify the reward for a single output through external criteria, whereas IBPO constructs process-sensitive shaping signals by comparing differences between successful and failed trajectories under the same prompt. The two are not mutually exclusive: rubric rewards can serve as terminal or intermediate evaluators, while IBPO's comparison operator $\mathcal{M}$ can further extract cross-trajectory credit-assignment information on top of those rewards.

SCoRe~\cite{kumar2025training} iteratively reuses previously generated responses and prompts the model to identify errors in earlier outputs. It improves reasoning accuracy through multiple rounds of reinforcement learning, but its training pipeline relies on multi-stage generation and optimization cycles. IBPO also includes correction generation, but this generation is used only to construct shaping signals and auxiliary correction objectives within a single round of joint training. All additional generation, verification, and comparison costs are included in the training-compute budget.

\section{Method}

\begin{figure}[t]
    \centering
    \resizebox{\linewidth}{!}{
    \begin{tikzpicture}[
        font=\scriptsize,
        >=Latex,
        box/.style={
            draw,
            rounded corners=2pt,
            align=center,
            minimum height=0.78cm,
            text width=2.25cm,
            inner sep=3pt,
            fill=blue!5
        },
        signal/.style={
            draw,
            rounded corners=2pt,
            align=center,
            minimum height=0.78cm,
            text width=2.45cm,
            inner sep=3pt,
            fill=orange!9
        },
        update/.style={
            draw,
            rounded corners=2pt,
            align=center,
            minimum height=0.78cm,
            text width=3.0cm,
            inner sep=3pt,
            fill=green!8
        },
        arr/.style={->, line width=0.45pt},
        darr/.style={->, dashed, line width=0.45pt}
    ]
        \node[box] (input) at (0, 0) {Input\\$x$};
        \node[box] (sample) at (2.9, 0) {Sample $G$ trajectories\\for the same input\\$\{\tau_i\}_{i=1}^{G}$};
        \node[box] (verify) at (5.8, 0) {Terminal verifier\\$R(\tau_i)$\\correct/incorrect};
        \node[box] (pair) at (8.7, 0) {Counterfactual pairing\\incorrect trajectory $\tau_i$\\alternative reference $\tau^{\mathrm{ref}}$};
        \node[signal, text width=3.05cm, minimum height=0.96cm] (compare) at (11.95, 0) {Comparison operator\\$\mathcal{M}$\\correction/ranking/consistency};

        \node[signal] (guard) at (11.75, -1.65) {Rule verification and\\rewrite prevention\\terminal reward + edit distance\\$\phi_i,\ \mathbf{m}_i$};
        \node[signal] (shape) at (8.65, -3.25) {Path 1: reward shaping\\$R'_i=R_i+\lambda\phi_i$\\within-group advantage $\widehat{A'}_i$};
        \node[signal] (mask) at (14.85, -3.25) {Path 2: token mask\\$\mathbf{m}_i$\\selective token updates};
        \node[update] (optimizer) at (11.75, -4.9) {Sequence-level optimizer\\GSPO / GRPO\\jointly trained policy $\pi_\theta$};

        \draw[arr] (input) -- (sample);
        \draw[arr] (sample) -- (verify);
        \draw[arr] (verify) -- (pair);
        \draw[arr] (pair) -- (compare);
        \draw[arr] (compare) -- (guard);
        \draw[arr] (guard) -- (shape);
        \draw[arr] (guard) -- (mask);
        \draw[arr] (shape) -- (optimizer);
        \draw[arr] (mask) -- (optimizer);
        \draw[darr] (optimizer.west) to[out=180,in=-90] node[pos=0.55, below, align=center] {Resample\\after update} (sample.south);
    \end{tikzpicture}
    }
    \caption{IBPO method diagram. For the same input, the policy first samples multiple candidate trajectories, and a terminal verifier provides sequence-level feedback. Incorrect trajectories are paired counterfactually with alternative reference trajectories; when a correct trajectory is available, it is used as the reference, otherwise an incorrect trajectory from the same group is used. The comparison operator $\mathcal{M}$ produces process-sensitive signals, which are passed through rule verification and rewrite-prevention filters before being injected into a standard sequence-level optimizer through sequence-level reward shaping or token-level gradient masking.}
    \label{fig:ibpo_method}
\end{figure}
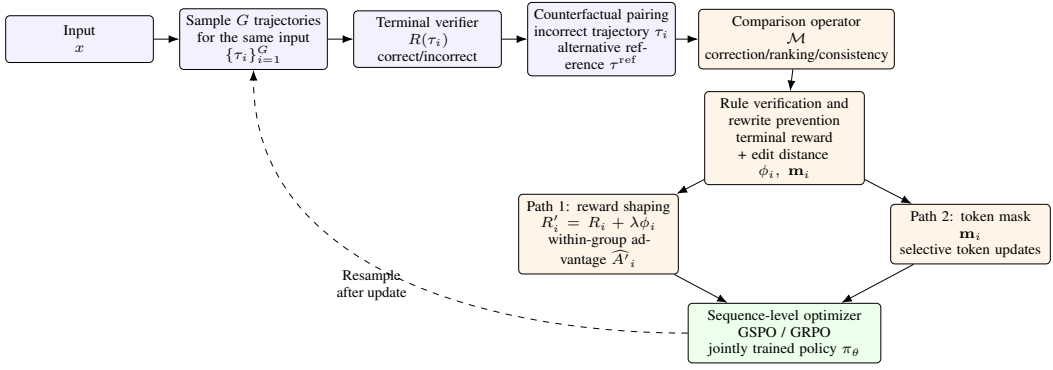

\subsection{Problem Formulation}

We consider multi-step reasoning reinforcement learning with terminal rewards. Given an input $x$, the policy $\pi_\theta$ generates a reasoning trajectory as a token sequence of length $T$:
\begin{equation}
    \tau = (y_1, y_2, \dots, y_T), \quad y_t \sim \pi_\theta(\cdot \mid x, y_{<t}),
\end{equation}
where $y_t$ denotes the token generated at step $t$.

The environment provides a \textbf{sequence-level reward} only when the trajectory is complete:
\begin{equation}
    R(\tau) \in [-1,1].
\end{equation}
In most reasoning tasks, $R(\tau)$ is sparse, for example the binary correctness of the final answer, and provides no explicit step-level supervision. To unify notation, the theoretical derivations in the main text use a symmetric binary reward $Y=R(\tau)\in\{-1,1\}$, whereas the experimental implementation and algorithm pseudocode use the correctness indicator $r(x,y)\in\{0,1\}$. The two are related by $Y=2r-1$, and all within-group normalized advantages preserve the same ranking structure under affine transformations.

The standard policy-gradient objective is:
\begin{equation}
    \nabla_\theta J(\theta)
    =
    \mathbb{E}_{\tau \sim \pi_\theta}
    \left[
    A(\tau)\sum_{t=1}^{T}
    \nabla_\theta \log \pi_\theta(y_t \mid x, y_{<t})
    \right].
\end{equation}

In multi-step reasoning tasks, the central challenge is not merely reward sparsity itself, but the fact that \textbf{terminal rewards assign credit to early decisions in an extremely unstable way}. When a local error occurs at an early step, its effect often cascades through and is amplified by later reasoning steps. Yet the error is reflected only indirectly through the terminal reward, producing highly noisy gradient signals whose variance can grow substantially with trajectory length.

\paragraph{IBPO as a framework rather than an implementation.}
We emphasize that IBPO is a \emph{training framework} for credit assignment under sparse terminal rewards, not a specific correction or rewriting algorithm. Its core contribution is to establish counterfactual trajectory comparison as a general mechanism for inducing implicit process-level learning signals. Concretely, the multi-trajectory comparison operator $\mathcal{M}$ in our framework is an abstract operator whose role is to extract cross-trajectory differences and generate learning signals that reflect differences in process-level decisions. The IBPO framework does not depend on implementation details such as how counterfactual differences are computed or what comparison mechanism is used. The operator $\mathcal{M}$ can be instantiated in many ways, such as consistency scoring, relative ranking, or error detection, but these are implementation details rather than constituent parts of the IBPO framework itself. Thus, IBPO's core contribution lies in its framework-level design, while concrete instantiations can be customized to task requirements.

Concrete comparison mechanisms such as correction, verifier-based ranking, or consistency scoring should be viewed as \emph{instantiations} of the comparison operator used within the IBPO framework. Our theoretical analysis and optimization framework apply to any instantiation that produces trajectory-dependent comparison signals sensitive to counterfactual differences.

\subsection{Counterfactual Trajectory Comparison}

\paragraph{Counterfactual trajectory comparison and the role of $\mathcal{M}$.}
The core idea of IBPO is to sample multiple reasoning trajectories for the same input and use their differences as counterfactual approximations for constructing process-sensitive learning signals. Specifically, we sample $G$ trajectories from the policy. Fully correct trajectories require no additional signal. For each incorrect target trajectory $\tau_i$, we pair it with $K\!-\!1$ alternative reference trajectories to form a target-centered $K$-tuple. When correct trajectories exist in the group, references are preferentially sampled from the correct trajectories; when there are too few correct trajectories, we replicate them to reach the required number. If no correct trajectory exists, we randomly select other incorrect trajectories as alternative references while excluding the target trajectory itself.
\begin{equation}
    \mathcal{T}_i =
    \left(\tau_i;\tau_{i,\mathrm{ref}}^{(1)},\ldots,\tau_{i,\mathrm{ref}}^{(K-1)}\right),
    \qquad K \ge 2.
\end{equation}

We introduce a \textbf{multi-trajectory comparison operator}
\begin{equation}
    \mathcal{M}(\mathcal{T}_i)=s_i
    \quad\text{or}\quad
    \mathcal{M}(\mathcal{T}_i)=(s_i,\mathbf{m}_i),
    \qquad
    s_i\in[0,1],\quad \mathbf{m}_i\in\{0,1\}^{T_i}.
\end{equation}
where $s_i$ is a comparison signal for the target trajectory $\tau_i$ and denotes its quality or recoverability relative to the reference trajectories, such as relative consistency, recoverability, or difference-aware quality. When a token-level variant is used, $\mathcal{M}$ also outputs a token mask $\mathbf{m}_i$ for the target trajectory $\tau_i$; otherwise only the scalar signal $s_i$ is used. The operator $\mathcal{M}$ can be implemented through various mechanisms and is subsequently verified by rule-based rewards to avoid circular reasoning and potential reward hacking or self-confirmation bias. The IBPO framework assumes only that the target-trajectory-dependent output of $\mathcal{M}$ is sensitive to the counterfactual differences between $\tau_i$ and its reference trajectories.

The output of $\mathcal{M}$ is not limited to sequence-level scalar signals. It can also produce token-level signals, such as the fraction of unchanged tokens used for reward shaping (IBPO-ratio), or a token-level mask used to block gradients on unchanged tokens (IBPO-mask). Detailed definitions of these variants are provided in Appendix~\ref{app:shaping_instance}.

\paragraph{Per-trajectory shaping function.}
Given the scalar comparison output $s_i$ from $\mathcal{M}(\mathcal{T}_i)$, we define the per-trajectory shaping function:
\begin{equation}
    \phi_i =
    \begin{cases}
        0 & \text{if } \tau_i \text{ is correct}, \\
        s_i \in [0, 1] & \text{otherwise}.
    \end{cases}
\end{equation}

$\phi$ is verified by rule-based rewards to avoid circular reasoning and potential reward hacking or self-confirmation bias. The function maps the comparison signal to a scalar shaping term for trajectory $\tau_i$. Importantly, $\phi_i$ depends on $\tau_i$ only through its relationship to other counterfactual trajectories, and it requires neither explicit step-level annotation nor value estimation.

\paragraph{Token-level masking.}
When the operator $\mathcal{M}$ can identify which tokens in a trajectory should and should not receive gradient updates, the comparison signal can be refined from the sequence level to a token-level mask. Let $\mathbf{m}_i = (m_{i,1}, \ldots, m_{i,T}) \in \{0, 1\}^T$ be the token-level mask produced by $\mathcal{M}$, where $m_{i,t}=1$ indicates that the $t$-th token should receive a gradient update and $m_{i,t}=0$ indicates that it should not be updated. The policy gradient can then selectively update only the specified tokens:
\begin{equation}
    \nabla_\theta \mathcal{J}_i^{\mathrm{mask}}
    =
    \frac{1}{M_i}
    \sum_{t=1}^{T} m_{i,t} \cdot \widehat{A'}_i \cdot \nabla_\theta \log \pi_\theta(y_t \mid y_{<t}, x),
    \qquad
    M_i=\max\!\left(1,\sum_{t=1}^T m_{i,t}\right).
    \label{eq:mask_grad}
\end{equation}
This masking mechanism concentrates gradient updates on potentially erroneous tokens and avoids unnecessary penalties on correct reasoning steps, thereby enabling finer token-level credit assignment. Concrete mask construction methods, such as edit-distance-based masking, are described in Appendix~\ref{app:shaping_instance}.

\subsection{Implicit Process-Level Advantage Estimation}

To inject comparison signals into optimization, we provide two complementary paths that replace the coarse feedback determined solely by $R(\tau)$.

\paragraph{Path 1: sequence-level reward shaping.}
For a candidate trajectory $\tau_i^{(k)}$, we define its shaped reward as:
\begin{equation}
    R'_i(x) = R(\tau_i) + \lambda \, \phi_i
    \label{eq:shaped_reward_general}
\end{equation}
where $0 \leq \lambda \phi_i < 1$, $\lambda \phi_i=0$ when $R(\tau_i)=1$, and $\lambda \phi_i$ is a positive value below 1 when $R(\tau_i)=-1$. Although $R'_i$ remains a sequence-level scalar, its value is conditioned on counterfactual comparisons among multiple trajectories and therefore statistically encodes process-level credit information.

We center $\lambda\phi_i$ through within-group advantage normalization:
\begin{equation}
    \widehat{A'}_i = \frac{R'_i(x) - \text{mean}\left( \{ R'_i(x) \}_{i=1}^G \right)}{\text{std}\left( \{ R'_i(x) \}_{i=1}^G \right)}.
\end{equation}
This normalization fixes the empirical scale of each within-group normalized advantage. Therefore, our theoretical analysis focuses on how the pre-normalization centered shaped rewards alter the within-group signal structure. The full post-normalization policy-gradient variance also depends on sequence-level log-probability gradients, trajectory length, and the correlation between the shaping signal and effective update directions.

\paragraph{Path 2: token-level gradient masking.}
When $\mathcal{M}$ produces a token-level mask $\mathbf{m}_i$ (as described above), the policy gradient can perform selective updates at token granularity (Eq.~\eqref{eq:mask_grad}). This path combines the sequence-level advantage $\widehat{A'}_i$ with a token-level mask: the advantage is still determined at the sequence level by counterfactual comparison, but gradients flow only through tokens marked for update. Neither path requires explicit step-level annotation or a value model.


\subsection{Mechanism: Positive Backward Transfer in Multi-Task Learning}
Counterfactual trajectory comparison does not directly provide explicit error labels to the model. Instead, by contrasting differences among multiple reasoning paths, latent errors become more salient during comparison. When this comparison behavior is jointly optimized with the base reasoning task, it induces \emph{positive backward transfer} from the auxiliary comparison task to the original reasoning task from the perspective of multi-task learning: the model learns to suppress local errors more quickly, thereby accelerating convergence on the base task. This mechanism reduces the number of ineffective updates needed to correct local errors and mitigates the \emph{learning tax} commonly observed in long-horizon reinforcement learning.

\paragraph{Testable prediction.}
This mechanism predicts especially large gains on difficult reasoning tasks, particularly when correct trajectories are extremely scarce and training signals are dominated by sparse terminal rewards. Concretely, we expect faster convergence and more stable training dynamics on challenging benchmarks. The process-level advantage-estimation mechanism and conditional variance analysis of IBPO are provided in Appendix~\ref{sec:variance_reduction}.


\section{Experiments}

\paragraph{Instantiating IBPO.}
As discussed above, IBPO provides a reward and advantage-construction framework based on multi-trajectory counterfactual comparison, rather than introducing a new sequence-level policy optimizer. Therefore, in concrete experiments, IBPO must be instantiated on top of an existing sequence-level reinforcement-learning method.

In this work, we instantiate IBPO on top of GSPO.
IBPO focuses on constructing process-sensitive advantage estimates from counterfactual multi-trajectory comparison under terminal rewards, while GSPO serves purely as the sequence-level optimizer that carries and applies these advantage signals. As shown in Appendix~\ref{sec:ibpospo1}, we provide a detailed IBPO+GSPO framework. We also observed similar trends on GRPO; for brevity, we omit those results. The concrete instantiation of the comparison operator $\mathcal{M}$ is described in Appendix~\ref{app:instantiation}.

\textbf{Tasks and datasets.}
We evaluate the proposed method on a collection of mathematical and code-reasoning benchmarks. The selected tasks are designed to assess symbolic manipulation, multi-step reasoning, domain-specific mathematical understanding, and code reasoning.

\textit{HMMT25}~\cite{NEURIPS2025_1d27c01e},
\textit{AIME25}~\cite{maa_aime2025} and
\textit{LiveCodeBench v6 (25.02-25.05)}~\cite{jain2025livecodebench}
\textbf{Base models.} Qwen3-32B~\cite{qwen3technicalreport}. Qwen3-Next-80B-A3B-Thinking~\cite{qwen3next}.

We configure Qwen3-32B with a 32k-token context length and Qwen3-Next-80B-A3B-Thinking with a 256k-token context length. Inference is performed with the VLLM engine (version 0.11.2).

\textbf{Baselines and comparison setup.}
We compare IBPO instantiated on GSPO~\cite{zheng2025group} (denoted \textbf{IBPO+GSPO}) with the following baselines: (1) vanilla GSPO; (2) \textbf{GSPO with prompt}, which introduces an additional correction prompt only at inference time after GSPO training to generate a revised output; (3) \textbf{GSPO with SCoRe}, which uses a SCoRe-style two-stage self-correction pipeline: it first generates an initial answer, then constructs a correction input from that answer and optimizes the corrected output. Its additional generation and optimization costs are counted under the same training-compute budget; and (4) \textbf{GSPO with Best-of-N}, which samples $N=8$ candidate answers for the same input and uses the same terminal verifier as the main experiments to select a candidate that passes verification. If multiple candidates pass, the first passing candidate is selected; if none pass, the first sample is retained. This baseline does not use oracle step-level labels or majority voting.

Experiments are conducted on 32 Nvidia A800 (80G) GPUs. Training hyperparameters are as follows: initial learning rate $5\times10^{-7}$; cosine-annealing learning-rate schedule with minimum learning-rate ratio 0.1; linear warmup for 3\% of the total training steps; entropy-regularization coefficient $\beta=0$; 64 rollouts per input for GSPO and 8 rollouts per input for IBPO+GSPO; and mini-batch size 32.

\paragraph{Compute-matching protocol.} We align the budgets of different methods by estimated training compute based on GPU runtime and GPU utilization. Details are provided in Appendix~\ref{sec:appendix-compute}.

\paragraph{Full-rewrite filtering.} To prevent the model from completely rewriting rather than locally repairing a solution during correction, which would lead to reward hacking, we detect full rewrites using edit distance and set their shaping rewards to zero. This mechanism is a defensive safeguard rather than a core component; see Appendix~\ref{app:rewrite_detection}.



Our experiments focus on mathematical and code-reasoning tasks. The applicability of the IBPO framework primarily depends on whether verifiable terminal feedback is available during training; further discussion is provided in Appendix~\ref{sec:appendix-scope}.

\section{Results and Analysis}

\begin{table}[t]
    \centering
    \renewcommand{\arraystretch}{1.1}
    \resizebox{\textwidth}{!}{
        \begin{tabular}{lcccccc}
            \toprule
            \multirow{2}{*}{\textbf{Method}} &
            \multicolumn{3}{c}{\textbf{Qwen3-32B}} &
            \multicolumn{3}{c}{\textbf{Qwen3-Next}} \\
            \cmidrule(lr){2-4}\cmidrule(lr){5-7}
            & \textbf{AIME25} & \textbf{LiveCodeBench} & \textbf{HMMT25}
            & \textbf{AIME25} & \textbf{LiveCodeBench} & \textbf{HMMT25} \\
            \midrule
            GSPO + IBPO
            & $85.3\pm1.2$ & $75.3\pm1.1$ & $62.6\pm1.4$
            & $93.8\pm1.5$ & $75.3\pm1.3$ & $80.4\pm1.6$ \\
            GSPO + IBPO(ratio)
            & $85.0\pm1.3$ & $74.9\pm1.2$ & $62.2\pm1.5$
            & $93.4\pm1.4$ & $74.7\pm1.4$ & $80.0\pm1.5$ \\
            GSPO + IBPO(mask)
            & \textbf{$85.9\pm1.1$} & \textbf{$76.0\pm1.0$} & \textbf{$63.1\pm1.3$}
            & \textbf{$94.4\pm1.3$} & \textbf{$75.8\pm1.2$} & \textbf{$81.1\pm1.4$} \\
            \midrule
            GSPO
            & $77.1\pm1.4$ & $64.6\pm1.5$ & $55.6\pm1.3$
            & $90.1\pm1.1$ & $70.9\pm1.7$ & $75.9\pm1.4$ \\
            GSPO with prompt
            & $78.2\pm0.8$ & $65.3\pm0.9$ & $56.4\pm1.6$
            & $90.6\pm1.2$ & $71.4\pm1.3$ & $76.5\pm1.8$ \\
            GSPO with SCoRe
            & $78.3\pm1.2$ & $65.2\pm1.2$ & $56.7\pm0.8$
            & $90.7\pm1.5$ & $71.9\pm1.2$ & $76.3\pm0.9$ \\
            GSPO with Best-of-N
            & $77.9\pm0.9$ & $66.1\pm1.7$ & $57.1\pm0.9$
            & $91.5\pm0.9$ & $71.6\pm1.5$ & $77.2\pm1.4$ \\
            \bottomrule
        \end{tabular}
    }
    \caption{Experimental results with Qwen3-32B and Qwen3-Next-80B-A3B-Thinking. For each test set, we conduct 64 independent pass@1 evaluations and report average accuracy; each evaluation samples one final answer per problem and uses neither majority vote nor pass@k. We report means over 5 random seeds with 95\% bootstrap confidence intervals (mean $\pm$ 95\% CI). The main method, GSPO+IBPO, is statistically significant against all baselines under a normal-approximation test ($p<0.01$). All methods are budget-aligned by the same training compute, and IBPO's base generation, correction generation, verification, and comparison costs are all included in the total budget. The Best-of-N method uses $N=8$.
    }
    \label{tab:scaled_result_qwen3}
\end{table}

Table~\ref{tab:scaled_result_qwen3} reports performance comparisons across multiple reasoning benchmarks. In addition to the base IBPO variant with sequence-level reward shaping, we evaluate two token-level edit-distance-based counterfactual analysis variants, IBPO(ratio) and IBPO(mask), whose detailed definitions appear in Appendix~\ref{app:shaping_instance}. Under matched training compute, GSPO+IBPO consistently outperforms GSPO, test-time prompt correction, SCoRe-style self-correction, and Best-of-N baselines across both base models and all three datasets; IBPO(mask) achieves the highest performance. Compared with GSPO, GSPO+IBPO improves Qwen3-32B by 8.2, 10.7, and 7.0 percentage points, and Qwen3-Next by 3.7, 4.4, and 4.5 percentage points.

These results show that the process-sensitive signals produced by counterfactual trajectory comparison improve the sample efficiency and stability of sequence-level RL under the same training-compute budget. Figure~\ref{fig:quxiaotu} further compares the training dynamics. The main experiments in Table~\ref{tab:scaled_result_qwen3} use a unified training-compute matching protocol: to avoid treating two-stage generation with different context lengths as having the same cost, we do not use the number of rollouts as the matching criterion. GSPO samples 64 base responses per input prompt, whereas IBPO samples 8 base responses per prompt and performs comparison-correction generation for incorrect responses. The two methods are aligned by unified training compute, and IBPO's base generation, correction generation, verification, rewrite filtering, and parameter-update costs are all included in the total budget. The training reward always refers to the batch average of the original terminal reward $r(x,y)$ rather than the shaped reward $r'(x,y)$; therefore, the reward curves for GSPO and IBPO in Fig.~\ref{fig:quxiaotu} use the same measurement.

\begin{figure}[t]
    \centering
    \includegraphics[width=0.9\linewidth]{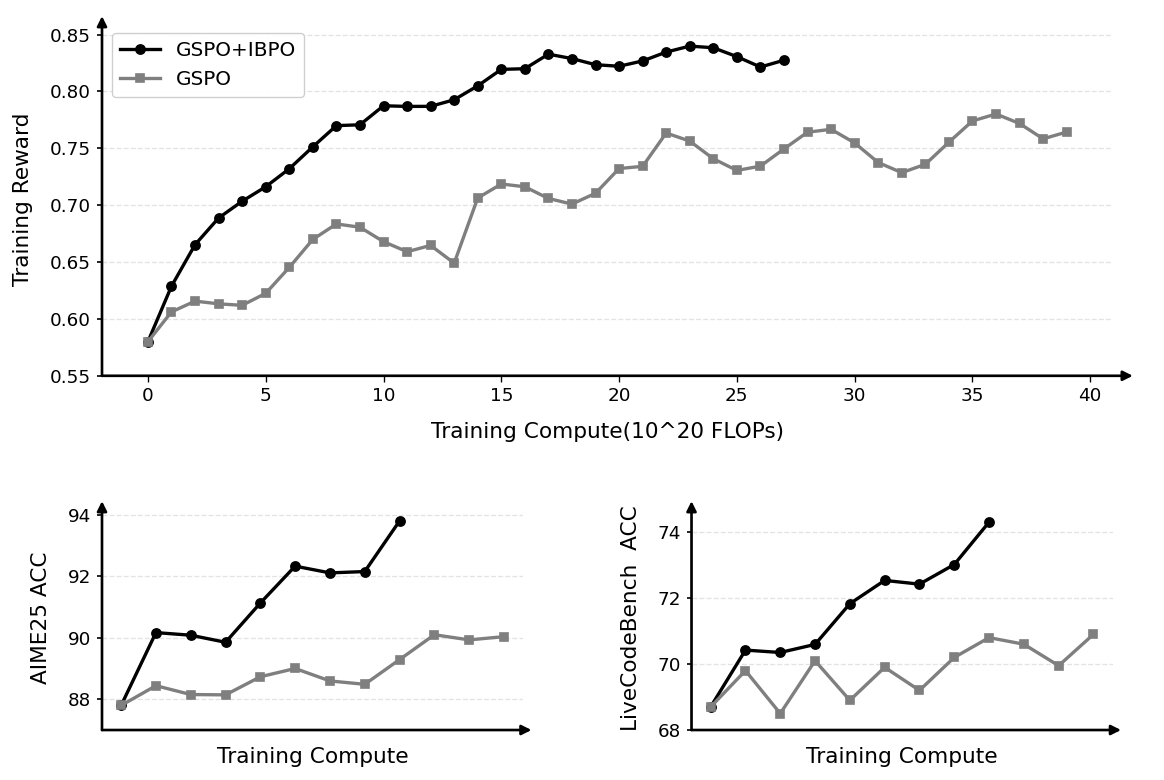}
    \caption{Training curves for fine-tuning Qwen3-Next-80B-A3B-Thinking. The horizontal axis is normalized estimated training compute; training reward is the batch average of the original terminal reward $r(x,y)$.}
    \label{fig:quxiaotu}
\end{figure}

\begin{table}[t]
    \centering
    \renewcommand{\arraystretch}{1.1}
    \begin{tabular}{l c}
    \toprule
    \textbf{Method} & \textbf{Compute@Reward=0.75} \\
    \midrule
    GSPO            & 1.00$\times$ \\
    GSPO + IBPO     & \textbf{0.63$\times$} \\
    \bottomrule
\end{tabular}
\caption{With Qwen3-Next-80B-A3B-Thinking, we measure the training budget required to reach a fixed training-reward threshold under a \textbf{training-compute matched} setting. Results are normalized relative to GSPO.}
\label{tab:compute_threshold}
\end{table}

Under budgets aligned by training compute, GSPO+IBPO consistently achieves higher training rewards throughout training and enters the high-reward regime earlier under the same compute constraint. Table~\ref{tab:compute_threshold} shows that GSPO+IBPO requires only 0.63$\times$ the training compute of GSPO to reach a fixed training-reward threshold of 0.75. This indicates that IBPO converts terminal rewards into more effective parameter updates \emph{per unit of compute}, improving training-compute efficiency.

Beyond average performance, GSPO+IBPO also exhibits a smoother training-reward trajectory with smaller fluctuations than GSPO. With only sequence-level terminal rewards, local reasoning errors are often propagated uniformly through the entire generated sequence via shared terminal feedback, causing gradient estimates to be affected by noise from irrelevant time steps. The shaping signals produced by counterfactual trajectory comparison help the model identify error-prone regions more quickly and suppress ineffective updates during optimization. The evolution of correction success rate, the positive-backward-transfer explanation, and a detailed comparison with test-time prompt correction are provided in Appendix~\ref{sec:appendix-training-dynamics}.

Ablation experiments further show that multi-trajectory counterfactual comparison, joint training, and the shaping weight all contribute to final performance. Full results and the $\lambda$ sensitivity analysis are given in Appendix~\ref{sec:appendix-ablation}.

\section{Conclusion}

We introduced Implicit Behavior Policy Optimization (IBPO), a reinforcement-learning paradigm that extracts implicit process-level learning signals from sparse terminal rewards through counterfactual comparison of reasoning trajectories. By sampling multiple trajectories per input and comparing their outcomes, IBPO enables more stable credit assignment without step-level supervision or an auxiliary value model.

Experiments show that, when combined with sequence-level optimizers such as GSPO, IBPO consistently improves performance and training stability on mathematical and code-reasoning benchmarks under matched training compute. Its formulation is agnostic to the underlying reinforcement-learning algorithm and is directly compatible with GRPO variants and other policy-gradient methods, offering a scalable path toward more robust multi-step reasoning in LLMs.

\section*{Limitations}

IBPO incurs additional compute overhead because it samples multiple counterfactual trajectories for each input. Although our results show higher sample efficiency under a fixed budget, reducing this overhead remains important for large-scale deployment. In addition, if the counterfactual trajectories contain systematic errors, the comparison signal may weaken. Increasing trajectory diversity or incorporating external verifiers can mitigate this issue.

Our experiments focus primarily on mathematics and code because these tasks have relatively reliable terminal verifiers. For open-ended question answering, writing, dialogue, and other settings that require LLM-as-a-Judge or rubric-model scoring, verifier noise directly affects recoverability judgments and shaping rewards, potentially amplifying erroneous preferences or weakening counterfactual comparison signals. Practical applications can reduce bias from noisy verifiers by using multi-judge consistency, confidence weighting, calibration-set filtering, or enabling shaping rewards only for high-confidence samples.

\section*{Societal Impact}

IBPO aims to improve the sample efficiency and training stability of reinforcement learning for multi-step reasoning. Potential positive impacts include lowering training-compute costs, improving the reliability of mathematical and code-reasoning systems, and reducing dependence on human step-level annotation. At the same time, stronger reasoning and code-generation capabilities could be used for unreliable automation, for generating incorrect but plausible explanations, or for assisting harmful code writing. Because the method depends on terminal verifiers, deployment to open-ended tasks may also inherit or amplify judge bias. Practical use should be paired with task risk assessment, verifier calibration, safety filtering, and human review.

\newpage
\appendix

\section{Scope of Applicability}
\label{sec:appendix-scope}

Although our experiments focus primarily on mathematical and code-reasoning tasks, the design and applicability of IBPO are not limited to a particular task domain. Its core mechanism is not multi-draft generation or explicit correction per se. Rather, IBPO constructs multiple counterfactual reasoning trajectories for the same input and uses their differences in terminal outcomes and intermediate decisions to induce learning signals that are more sensitive to the reasoning process. From this perspective, IBPO is fundamentally a training paradigm based on counterfactual trajectory comparison. Its applicability depends only on whether some objective or verifiable feedback signal is available during training, rather than on a particular task format or output structure.

In mathematics and programming tasks, the correctness of the final solution has a clear and automatically verifiable definition, making these tasks convenient and reliable testbeds for studying the role of counterfactual trajectory comparison in multi-step reasoning RL. For other task types, such as factual question answering, structured knowledge reasoning, or multi-step decision problems with clear termination conditions, one can likewise design appropriate verifiable reward functions or evaluation criteria to distinguish the terminal quality of different counterfactual trajectories. By comparing multiple counterfactual trajectories sampled for the same input during training and injecting the resulting differences into reward shaping or advantage estimation, the model can statistically learn which intermediate decisions are more likely to lead to success and which are more likely to lead to failure.

Structurally, the key advantage of IBPO lies in its modeling of counterfactual trajectory independence and the resulting implicit process-level credit-assignment mechanism. This mechanism does not rely on explicit step-level annotations or an additional value model. Instead, with only terminal rewards, it introduces more discriminative learning signals for the reasoning process through multi-trajectory counterfactual comparison. Therefore, although we do not yet provide empirical results on non-mathematical or non-code tasks, the principles of counterfactual trajectory comparison and process-level advantage construction underlying IBPO are, in principle, applicable to any reasoning or decision task with verifiable terminal outcomes.

\section{Existing Assets and Licenses}
\label{sec:appendix-assets}

The existing models, datasets, and evaluation tools used in this paper are all taken from public release pages. Based on the official model cards, dataset cards, or code repositories, we record the following license information:

\begin{itemize}
    \item \textbf{Qwen3-32B}: base model. The official Hugging Face page is \url{https://huggingface.co/Qwen/Qwen3-32B}, and the license is marked as Apache-2.0.
    \item \textbf{Qwen3-Next-80B-A3B-Thinking}: base model. The official Hugging Face page is \url{https://huggingface.co/Qwen/Qwen3-Next-80B-A3B-Thinking}, and the license is marked as Apache-2.0. The official Qwen repository also states that its open-weight models use the Apache 2.0 license.
    \item \textbf{MathArena/aime\_2025}: AIME 2025 evaluation data. The official Hugging Face dataset page is \url{https://huggingface.co/datasets/MathArena/aime_2025}, and the license is marked as CC BY-NC-SA 4.0. The AoPS AIME 2025 page also states that the original problems are copyrighted by the Mathematical Association of America.
    \item \textbf{MathArena/hmmt\_feb\_2025}: HMMT February 2025 evaluation data. The official Hugging Face dataset page is \url{https://huggingface.co/datasets/MathArena/hmmt_feb_2025}, and the license is marked as CC BY-NC-SA 4.0.
    \item \textbf{LiveCodeBench}: code-evaluation benchmark. The official GitHub repository is \url{https://github.com/LiveCodeBench/LiveCodeBench}, and the code license is MIT.
    \item \textbf{vLLM}: inference engine. The official GitHub repository is \url{https://github.com/vllm-project/vllm}, and the license is Apache-2.0.
\end{itemize}

Beyond the public assets listed above, this paper does not release a new dataset or benchmark. If training code or checkpoints are released later, the corresponding licenses, terms of use, and reproduction instructions will be added to the release materials.

\section{Additional Ablation Experiments}
\label{sec:appendix-ablation}

To evaluate the contribution of each IBPO component, we conduct ablation experiments with Qwen3-32B on AIME25, LiveCodeBench, and HMMT25. Table~\ref{tab:ablation} summarizes the results.

\begin{table}[H]
    \centering
    \renewcommand{\arraystretch}{1.1}
    \resizebox{0.5\textwidth}{!}{
        \begin{tabular}{lcccc}
            \toprule
            \textbf{Model Variant} & \textbf{AIME25 (\%)} & \textbf{LiveCodeBench (\%)} & \textbf{HMMT25 (\%)} &  \\
            \textbf{Full IBPO($K=2$)}    & \textbf{85.3} & \textbf{75.3} & \textbf{62.6} &  \\
            \hline
            GSPO (baseline)               & 77.1 & 64.6 & 55.6  \\
            GSPO + test-time prompt       & 78.2 & 65.3 & 56.4  \\
            IBPO($K=1$)                   & 78.6 & 65.9 & 57.1  \\
            IBPO (shaping only)           & 80.3 & 70.2 & 59.7  \\
            \bottomrule
        \end{tabular}
    }
    \caption{Ablation results for Qwen3-32B on AIME25, LiveCodeBench, and HMMT25. Each variant removes one key component from the full IBPO algorithm.}
    \vspace{6pt}\label{tab:ablation}
\end{table}

\paragraph{GSPO + test-time prompt.}
Multi-trajectory comparison is performed only at inference time through prompting. Without joint training, the gains introduced by IBPO disappear, leading to a substantial performance drop. This result validates the \emph{positive transfer} induced by IBPO and the effectiveness of implicit process-level rewards.

\paragraph{IBPO ($K=1$).}
This variant uses only the target answer itself to construct the correction input and provides no alternative reference from other candidate trajectories under the same prompt. It therefore degenerates into a self-correction-only training setup: the model can still attempt to revise its own output, but it cannot use cross-trajectory differences to localize local errors or construct relational counterfactual signals. Accuracy drops substantially without comparison among multiple counterfactual trajectories, indicating that inconsistencies among such trajectories play a key role in error identification.

\paragraph{IBPO (shaping only).}
In this ablation, we disable joint training for multi-trajectory comparison while retaining the reward-shaping term. This isolates the positive-transfer effect brought by joint training.

\begin{table}[h!]
    \centering
    \resizebox{\columnwidth}{!}{%
        \begin{tabular}{cccccc}
            \hline
            \textbf{$\lambda$} & \textbf{0.4} & \textbf{0.6} & \textbf{0.8} & \textbf{1.0} & \textbf{1.2} \\ \hline
            Accuracy (\%)      & 83.1         & \textbf{85.3} & 83.6         & 82.1         & 81.5         \\ \hline
        \end{tabular}%
    }
    \caption{Evaluation results under different $\lambda$ values (IBPO score; Qwen3-32B; AIME25).}
    \label{tab:lambda_sensitivity}
\end{table}

We conduct a sensitivity analysis over $\lambda$ and observe the best performance around 0.6. Overall, the ablation results clearly show that each IBPO component makes a meaningful contribution to the final performance.

\section{Training Dynamics and Additional Result Analysis}
\label{sec:appendix-training-dynamics}

\paragraph{Evolution of correction success rate.}
The informativeness of the shaping signal depends on whether the correction success rate remains in a meaningful intermediate range: if correction almost always succeeds or almost always fails, the signal degenerates into a constant. We track the evolution of correction success rate during Qwen3-32B training on AIME25. The correction success rate is about 12\% early in training and gradually rises to about 67\% as training progresses. This indicates that the shaping signal remains in an information-rich regime throughout training, neither constantly zero nor constantly one, and thus continues to provide discriminative process-level feedback for policy optimization.

\subsection{Positive Backward Transfer and Difficult Tasks}
As shown in Fig.~\ref{fig:quxiaotu}, introducing IBPO leads to faster performance improvement and faster convergence. We attribute this to a \emph{positive backward transfer} effect. In multi-task learning, positive backward transfer refers to the phenomenon in which learning a later task (task~B) improves performance on an earlier task (task~A), reflecting strong generalization. By introducing an auxiliary task based on counterfactual comparison of reasoning trajectories, IBPO induces a significant positive transfer effect on the main reasoning task during training.

Concretely, under GSPO training the model receives only a sequence-level reward $R(y)$, which is propagated uniformly across the entire reasoning sequence. When final failure is caused by only a few local tokens, this supervision signal cannot indicate where the error occurred, forcing the model to rely on extensive sampling and iterative updates to gradually internalize these local errors statistically. This substantially increases the sample complexity of learning and gives rise to the \emph{learning tax} commonly observed in reinforcement learning. IBPO introduces an auxiliary task based on counterfactual comparison of reasoning trajectories. By contrasting inconsistencies across different reasoning paths, it increases the \emph{observability} of local errors and accelerates the model's internalization of fine-grained reasoning mistakes.

In difficult reasoning tasks, correct responses are often extremely rare, and the vast majority of model-generated trajectories are incorrect. This distribution makes sequence-level rewards highly sparse and causes policy-gradient estimates to be dominated by negative samples, slowing convergence and potentially destabilizing training. In contrast, without changing the original sampling distribution, IBPO converts a small number of correct reasoning paths into multiple informative learning signals through counterfactual trajectory comparison, thereby alleviating the learning bottleneck caused by positive-sample scarcity in difficult tasks.

\subsection{Comparison with Prompted GSPO}

The core difference between GSPO+prompt and IBPO+GSPO is whether joint training is performed. The former conducts multi-trajectory comparison and correction through prompting only at inference time after training is complete, whereas the latter performs multi-trajectory comparison during training and jointly optimizes the model. The experimental results validate the effectiveness of implicit process-level rewards. Overall, these results show that IBPO not only improves overall accuracy but also substantially strengthens model robustness across problems of varying difficulty.

\section{Theoretical Analysis: Conditional Variance-Reduction Properties of IBPO}
\label{sec:variance_reduction}

To characterize the role of IBPO in credit assignment, we use representative GSPO-style methods as a baseline and analyze how counterfactual comparison signals affect the variance of within-group centered reward terms. The analysis is conditional: it does not claim that IBPO unconditionally reduces the full policy-gradient variance in all settings, nor that the sample-advantage variance after standard-deviation normalization is further reduced. Instead, it shows that when the comparison signal complements terminal-reward noise and the shaping weight is moderate, the pre-normalization advantage terms constructed by IBPO can have lower within-group dispersion.

\subsection{Process-Level Advantage-Estimation Mechanism}

\begin{proposition}[IBPO's process-level advantage-estimation mechanism]
    For any multi-step reasoning task, suppose the policy $\pi_\theta$ generates $G$ trajectories $\{\tau_i\}_{i=1}^G$ for a given input $x$, and each trajectory $\tau_i$ receives a sequence-level reward $R(\tau_i)$. For each incorrect target trajectory $\tau_i$, a target-trajectory-dependent comparison signal $s_i$ is generated by counterfactual comparison with $K-1$ alternative reference trajectories; these reference trajectories are correct trajectories when available and otherwise other incorrect trajectories in the group. The comparison signal is injected into optimization through two paths: (i) it is mapped to a shaped reward $R'_i(x)$; and (ii) when $\mathcal{M}$ produces a token-level mask $\mathbf{m}_i$, only the marked tokens are selectively updated.

    When the shaping function $\phi_i$ provides informative process signals for suboptimal trajectories and this signal complements terminal-reward noise, IBPO's advantage estimate can reduce the within-group dispersion of pre-normalization centered reward terms and improve the update-signal structure induced by coarse episode-level rewards. Specifically, in Path 1, $\phi_i$ provides finer process feedback through reward shaping; in Path 2, the token-level mask further concentrates gradient updates on potentially erroneous tokens and avoids unnecessary penalties on correct reasoning steps. These two paths provide a mechanistic explanation for more stable training dynamics.

\end{proposition}

\begin{proof}
    By introducing the multi-trajectory comparison operator $\mathcal{M}$, we obtain target-trajectory-dependent comparison signals $s_i(x)$ from the $G$ trajectories. Each signal is sensitive to the counterfactual differences between the target trajectory and its reference trajectories. These differences reflect the effects of different generated tokens in the reasoning process and are injected into optimization through two paths: (i) $\phi_i$ is mapped to the shaped reward $R'_i(x)$, enabling finer sequence-level credit assignment; and (ii) the token-level mask $\mathbf{m}_i$ restricts gradient updates to modified tokens (Eq.~\eqref{eq:mask_grad}), enabling selective token-granular credit assignment. Thus, IBPO does not rely on an unconditional variance-reduction guarantee. Instead, it uses more informative comparison signals to reduce interference from irrelevant tokens and irrelevant trajectories.

    Concretely, $\phi_i$ provides process-level feedback for incorrect trajectories rather than relying only on terminal rewards; $\mathbf{m}_i$ further filters out gradient contributions from correct tokens, making update signals more precise. The following subsection gives a conditional finite-sample variance analysis for the reward-shaping path. The token-masking path is defined as one possible output form of $\mathcal{M}$ in Appendix~\ref{app:shaping_instance} and is integrated into GSPO through the masked ratio in Appendix~\ref{sec:ibpo_gspo}.
\end{proof}

\subsection{Conditional Variance Reduction for Centered Advantages}

We consider a group of trajectories $\{\tau_i\}_{i=1}^G$ sampled from the policy $\pi_\theta$ for a fixed input $x$. To make the finite-sample variance calculation clear, we analyze an idealized but standard setting: conditional on $x$, $(Y_i,\phi_i)$ are exchangeable within the group and share the same marginal variances and covariances. This assumption allows $\phi_i$ to have weak dependence on other samples through within-group reference selection. Therefore, the conclusion below should be understood as a conditional analysis of the direction in which the IBPO shaping term acts, not as an unconditional guarantee over all implementation details. Let:
\begin{itemize}
    \item $Y_i = R(\tau_i) \in \{-1, 1\}$ denote the terminal reward of the $i$-th trajectory; for simplicity, we assume a binary reward. The implementation uses the correctness reward $r(x,y_i)\in\{0,1\}$; Appendix~\ref{sec:variance_reduction} uses the affine recoding $Y_i=2r(x,y_i)-1$ only to simplify covariance notation.
    \item $\phi_i \in [0, 1]$ be the counterfactual comparison signal introduced by IBPO, satisfying:
    \[
    \phi_i =
    \begin{cases}
        0, & \text{if } Y_i = 1 \ (\text{correct trajectory}); \\
        >0, & \text{if } Y_i = -1 \ (\text{incorrect trajectory}).
    \end{cases}
    \]
    In addition, we assume that $\phi_i$ effectively reflects the trajectory's ``recoverability'' or ``consistency with correct reasoning'': the closer an incorrect trajectory is to the correct reasoning process, the larger $\phi_i$ is.
\end{itemize}

Based on this, GSPO and IBPO define the following within-group centered advantage terms; the standard-deviation normalization factor is omitted for now and discussed later:
\begin{align}
    A^{\mathrm{GSPO}}_i &= Y_i - \bar{Y}, \quad \text{where } \bar{Y} = \frac{1}{G} \sum_{j=1}^G Y_j, \\
    A^{\mathrm{IBPO}}_i &= (Y_i + \lambda \phi_i) - (\bar{Y} + \lambda \bar{\phi}) = A^{\mathrm{GSPO}}_i + \lambda (\phi_i - \bar{\phi}),
\end{align}
where $\lambda > 0$ is the shaping weight and $\bar{\phi} = \frac{1}{G} \sum_{j=1}^G \phi_j$.

For the general exchangeable setting, we use the following structural condition:

\begin{condition}[Effective negative correlation and non-degeneracy]
    \label{cond:negative_corr}
    Let
    \[
        C_{\mathrm{in}}=\mathrm{Cov}(Y_i,\phi_i),
        \qquad
        C_{\mathrm{out}}=\mathrm{Cov}(Y_i,\phi_j),\quad j\ne i .
    \]
    The terminal reward $Y_i$ and comparison signal $\phi_i$ satisfy $C_{\mathrm{in}}-C_{\mathrm{out}}<0$. In addition, let
    \[
        V_{\mathrm{in}}=\mathrm{Var}(\phi_i),
        \qquad
        V_{\mathrm{out}}=\mathrm{Cov}(\phi_i,\phi_j),\quad j\ne i ,
    \]
    and require $V_{\mathrm{in}}-V_{\mathrm{out}}>0$, meaning that the comparison signal is not constant within the group. When cross-sample correlation introduced by within-group reference selection is negligible, $C_{\mathrm{out}}\approx0$, and the condition reduces to $\mathrm{Cov}(Y_i,\phi_i)<0$. Intuitively, correct trajectories ($Y_i=1$) force $\phi_i=0$, whereas incorrect trajectories ($Y_i=-1$) have $\phi_i>0$, and larger $\phi_i$ indicates closer proximity to correct reasoning.
\end{condition}

\begin{lemma}[Design-induced negative correlation without cross-sample dependence]
    \label{lem:design_negative_corr}
    Suppose the comparison signal depends only on the target trajectory, so there is no cross-sample dependence and $\mathrm{Cov}(Y_i,\phi_j)=0$ for $j\ne i$. If $0<\mathbb{P}(Y_i=1)<1$, $\phi_i=0$ when $Y_i=1$, and $\mathbb{E}[\phi_i\mid Y_i=-1]>0$, then
    \[
        \mathrm{Cov}(Y_i,\phi_i)<0.
    \]
    Therefore, the negative-correlation part of Condition~\ref{cond:negative_corr} holds when $C_{\mathrm{out}}=0$.
\end{lemma}

\begin{proof}
    Let $p=\mathbb{P}(Y_i=1)$ and $m=\mathbb{E}[\phi_i\mid Y_i=-1]$. Since $Y_i\in\{-1,1\}$ and $\phi_i=0$ on correct trajectories, we have
    \[
        \mathbb{E}[Y_i\phi_i]=-(1-p)m,\quad
        \mathbb{E}[Y_i]=2p-1,\quad
        \mathbb{E}[\phi_i]=(1-p)m .
    \]
    Therefore,
    \[
        \mathrm{Cov}(Y_i,\phi_i)
        =\mathbb{E}[Y_i\phi_i]-\mathbb{E}[Y_i]\mathbb{E}[\phi_i]
        =-2p(1-p)m<0 .
    \]
\end{proof}

\begin{proposition}[Conditional variance reduction for centered advantages]
    \label{prop:conditional_variance_reduction}
    Under Condition~\ref{cond:negative_corr} and group size $G \geq 2$, there exists $\lambda_{\max} > 0$ such that for any $\lambda \in (0, \lambda_{\max})$:
    \[
    \mathrm{Var}\left( A^{\mathrm{IBPO}}_i \right) < \mathrm{Var}\left( A^{\mathrm{GSPO}}_i \right).
    \]
\end{proposition}

\begin{proof}
    From $A^{\mathrm{IBPO}}_i = A^{\mathrm{GSPO}}_i + \lambda (\phi_i - \bar{\phi})$, expanding the variance gives:
    \begin{align}
        \mathrm{Var}(A^{\mathrm{IBPO}}_i)
        &= \mathrm{Var}\big( A^{\mathrm{GSPO}}_i + \lambda (\phi_i - \bar{\phi}) \big) \nonumber \\
        &= \mathrm{Var}(A^{\mathrm{GSPO}}_i)
        + \lambda^2 \mathrm{Var}(\phi_i - \bar{\phi})
        + 2\lambda \, \mathrm{Cov}\big( A^{\mathrm{GSPO}}_i,\, \phi_i - \bar{\phi} \big).
    \end{align}
    Note that $A^{\mathrm{GSPO}}_i = Y_i - \bar{Y}$. Because the within-group samples are exchangeable conditional on the input $x$, the finite-sample covariance can be computed exactly without taking $G$ to infinity. Let
    \[
        C_{\mathrm{in}}=\mathrm{Cov}(Y_i,\phi_i),
        \qquad
        C_{\mathrm{out}}=\mathrm{Cov}(Y_i,\phi_j),\quad j\ne i ,
    \]
    Then:
    \[
    \mathrm{Cov}\big( A^{\mathrm{GSPO}}_i,\, \phi_i - \bar{\phi} \big)
    =
    \mathrm{Cov}\big( Y_i-\bar{Y},\, \phi_i - \bar{\phi} \big)
    =
    \left(1-\frac{1}{G}\right)(C_{\mathrm{in}}-C_{\mathrm{out}}) < 0,
    \]
    where the strict inequality is guaranteed by Condition~\ref{cond:negative_corr}. Similarly, let
    \[
        V_{\mathrm{in}}=\mathrm{Var}(\phi_i),
        \qquad
        V_{\mathrm{out}}=\mathrm{Cov}(\phi_i,\phi_j),\quad j\ne i ,
    \]
    Then
    \[
        \mathrm{Var}(\phi_i-\bar{\phi})
        =
        \left(1-\frac{1}{G}\right)(V_{\mathrm{in}}-V_{\mathrm{out}}).
    \]
    By the non-degeneracy part of Condition~\ref{cond:negative_corr}, $V_{\mathrm{in}}-V_{\mathrm{out}}>0$; otherwise, the shaping term would not provide distinguishable within-group information.

    Let $C=-(C_{\mathrm{in}}-C_{\mathrm{out}})>0$ and $V_\phi=V_{\mathrm{in}}-V_{\mathrm{out}}>0$. Substituting into the expression above yields:
    \[
    \mathrm{Var}(A^{\mathrm{IBPO}}_i)
    =
    \mathrm{Var}(A^{\mathrm{GSPO}}_i)
    +
    \left(1-\frac{1}{G}\right)
    \left(
    \lambda^2 V_\phi - 2\lambda C
    \right).
    \]
    Because $G\ge2$ and $1-\frac{1}{G}>0$, when
    \[
        0<\lambda<\lambda_{\max}
        =
        \frac{2C}{V_\phi}
    \]
    the quadratic term in parentheses is strictly negative, and thus
    $\mathrm{Var}(A^{\mathrm{IBPO}}_i)<\mathrm{Var}(A^{\mathrm{GSPO}}_i)$.

    This conclusion characterizes only the within-group centered advantage term before standard-deviation normalization. GSPO implementations typically also divide by the empirical within-group standard deviation, so the sample variance after normalization is not itself the object explained by this theoretical result. The full policy-gradient variance also depends on the correlation between $\nabla_\theta \log \pi_\theta(\tau_i \mid x)$ and the shaping signal, trajectory length, the sampling distribution, and optimizer implementation. Therefore, the proposition should be understood as a consistency analysis showing how the IBPO shaping term improves the structure of update signals, rather than as an unconditional guarantee on the full policy-gradient variance.
\end{proof}

\paragraph{Discussion.}
The proposition shows that, under Condition~\ref{cond:negative_corr}, the shaping term $\lambda \phi_i$ introduced by IBPO's counterfactual comparison can reduce the variance of pre-normalization centered advantage terms. It is important to emphasize that the effective negative correlation here is not an independently discovered empirical law, but a structural condition partially induced by the design of the shaping term: correct trajectories typically do not require recoverability rewards, while incorrect but repairable trajectories receive additional shaping signals. Lemma~\ref{lem:design_negative_corr} explicitly demonstrates this in the target-trajectory-local case without cross-sample dependence. The $C_{\mathrm{out}}$ term in Condition~\ref{cond:negative_corr} extends the statement to a more general exchangeable setting, capturing weak dependence that may arise from within-group reference selection. Thus, the proposition does not prove that IBPO necessarily reduces variance in all settings. Rather, it shows that when the shaping signal genuinely distinguishes recoverable from unrecoverable errors and $\lambda$ is not too large, the IBPO advantage estimate is not merely adding arbitrary noise, but has a plausible mechanism for improving the structure of within-group update signals. Changes in the full gradient variance still need to be measured empirically.

\paragraph{Relation to training dynamics.}
Proposition~\ref{prop:conditional_variance_reduction} provides a directional explanation, not an unconditional guarantee on the full policy-gradient variance. It shows that when the shaping signal truly distinguishes recoverable from unrecoverable errors and the shaping weight is not too large, IBPO's advantage estimate need not merely add arbitrary noise and may instead improve the within-group update-signal structure. The smoother reward evolution and faster convergence in Fig.~\ref{fig:quxiaotu} are consistent with this mechanistic explanation.

\section{Details of Compute-Budget Matching}
\label{sec:appendix-compute}
To ensure fair comparison, we match the total training-compute budget of different methods by considering: (1) the number of sampled trajectories and (2) the total compute cost. Therefore, we compare performance only under the same training-compute budget.

In our experiments, the compute budgets of IBPO+GSPO and GSPO are matched through a unified training-compute estimate. Specifically, for each prompt $x$, IBPO+GSPO first generates 8 responses $y$. For each incorrect response $y_i$, it concatenates $y_i$ with the original input $x$ and a reference response to form a new input, then generates a corrected output. When a correct response exists in the group, the reference is preferentially sampled from correct responses; otherwise, it is randomly sampled from other incorrect responses in the group, excluding the target response $y_i$ itself. The correction-stage input sequence is longer because it concatenates additional context, and the quadratic complexity of attention makes each correction trajectory more expensive than base sampling. Therefore, we do not use rollout count as the budget-alignment criterion. The horizontal axis in Fig.~\ref{fig:quxiaotu} corresponds to this training-compute budget.

\section{Instantiation Details of IBPO}
\label{app:instantiation}

This appendix presents one concrete instantiation of IBPO used in our experiments, namely the \emph{compare-and-correct} mechanism, along with the integrated training pipeline and implementation details when combined with GSPO. We emphasize that the following design is one concrete choice for the comparison operator $\mathcal{M}$ in the main text; the core formulation of IBPO does not depend on this particular instantiation.

\subsection{Operator}
\label{app:correction_operator}

\paragraph{Instantiation choice.}
In our experiments, we instantiate the general comparison operator $\mathcal{M}$ with a \emph{compare-and-correct} mechanism. We emphasize that this is \emph{one concrete implementation} of IBPO, chosen for its simplicity and effectiveness on verifiable reasoning tasks, rather than a requirement of the IBPO formulation itself.

Concretely, we first generate multiple candidate reasoning trajectories, then use the model itself to compare these trajectories and rewrite them into corrected outputs. This implementation can be viewed as one concrete choice of $\mathcal{M}$ that maps counterfactual differences to a computable shaping term $\phi(\cdot)$. To avoid tying the contribution of this work to a specific implementation, we defer details such as how trajectory diversity is induced and how comparison inputs are constructed to this appendix.

Given two candidate reasoning trajectories/responses for the same input $x$, namely the target response $y$ and the reference response $y^{\mathrm{ref}}$, we construct a correction input $\tilde{x}=(x; y, y^{\mathrm{ref}})$ and let the model generate a revised output conditioned on $\tilde{x}$:
\begin{equation}
    \hat y \sim \pi_\theta(\cdot \mid \tilde{x}), \qquad
    \hat y = \mathcal{C}(x; y, y^{\mathrm{ref}}).
\end{equation}

\paragraph{Prompt template.}
In all experiments, we use the following compare-and-correct instruction template, with minor adjustments for task format when needed:
\begin{quote}
    You are given two candidate solutions to the same problem.
    Compare them step by step, identify any inconsistencies or errors,
    and then produce a corrected solution and final answer.
\end{quote}

\paragraph{Reference sampling.}
For each target response $y_i$ to be corrected, we sample the reference response $y_i^{\mathrm{ref}}$ as follows. If the group contains at least one correct response, we sample uniformly from the set of correct responses. Otherwise, we sample uniformly from the set of other incorrect responses $\{y_j:r(x,y_j)=0,\ j\ne i\}$. Thus, the fallback reference response is never the target response itself. This strategy is designed to provide a relatively stronger, or at least different, counterfactual reference for comparison without introducing external supervision.

\subsection{Shaping Instance: Recoverability-Induced Reward}
\label{app:shaping_instance}

We use a shaping instance based on \emph{recoverability} to define $\phi(\cdot)$. Let $r(x,y)\in\{0,1\}$ denote the terminal correctness reward, i.e., whether the final answer is correct. For an incorrect response $y$ and its corrected output $\hat y=\mathcal{C}(x;y,y^{\mathrm{ref}})$, we define the implicit process-shaping term as:
\begin{equation}
    \Delta(x;y,y^{\mathrm{ref}})
    =
    \rho\cdot \mathbb{I}\!\left[r(x,y)=0 \ \wedge\ r(x,\hat y)=1\right],
    \qquad \rho=0.5.
    \label{eq:app_delta}
\end{equation}
The resulting shaped sequence-level reward is
\begin{equation}
    r'(x,y) = r(x,y) + \lambda\,\Delta(x;y,y^{\mathrm{ref}}).
    \label{eq:app_shaped_reward}
\end{equation}

\paragraph{Token-level edit-distance variants.}
When correction succeeds, i.e., $r(x,y)=0 \wedge r(x,\hat y)=1$, we can further localize the modified tokens by computing the token-level edit distance between the original response $y$ and the corrected output $\hat y$. Let $y=(y_1,\ldots,y_T)$ and $\hat y=(\hat y_1,\ldots,\hat y_T)$ after alignment, and define the set of unchanged tokens as $\mathcal{U}=\{t: y_t = \hat y_t\}$.

\textbf{IBPO-ratio} (unchanged-ratio reward shaping): use the fraction of unchanged tokens $|\mathcal{U}|/T$ as the recoverability measure, replacing the fixed scale $\rho$ in Eq.~\eqref{eq:app_delta}:
\begin{equation}
    \Delta^{\mathrm{ratio}}(x;y,y^{\mathrm{ref}})
    =
    \frac{|\mathcal{U}|}{T}\cdot \mathbb{I}\!\left[r(x,y)=0 \ \wedge\ r(x,\hat y)=1\right].
    \label{eq:app_delta_ratio}
\end{equation}
Intuitively, a higher unchanged ratio indicates that the original reasoning is closer to being correct and should receive a larger shaping reward.

\textbf{IBPO-mask} (token-level gradient mask): mask the gradient contribution of unchanged tokens and apply policy-gradient updates only to modified tokens. Define the token-level mask $m_t = \mathbb{I}[t \notin \mathcal{U}]$; then the gradient coefficient for the $t$-th token in the policy gradient is multiplied by $m_t$:
\begin{equation}
    \nabla_\theta \mathcal{J}^{\mathrm{mask}}
    \propto
    \frac{1}{M}
    \sum_{t=1}^{T} m_t \cdot \widehat A(y) \cdot \nabla_\theta \log \pi_\theta(y_t \mid y_{<t}, x),
    \qquad
    M=\max\!\left(1,\sum_{t=1}^T m_t\right),
    \label{eq:app_mask}
\end{equation}
Here $\widehat A(y)$ is the within-group normalized advantage obtained from the sequence-level reward or shaped reward and broadcast to the selected tokens. It is not a token-level value model and is not an advantage estimated token by token. Unchanged tokens ($m_t=0$) are not penalized because they are preserved after correction, suggesting that they may be correct reasoning steps. Its integration with GSPO is given in the masked-ratio form in Appendix~\ref{sec:ibpo_gspo}.

\paragraph{Why not penalize failed correction.}
When $r(x,y)=0$ and $r(x,\hat y)=0$, we do not apply an additional negative penalty. This avoids misattributing insufficient correction ability or poor reference quality to the intrinsic quality of the original reasoning. This design choice helps prevent unnecessary bias and training instability.

\subsection{Full-Rewrite Detection and Suppression}
\label{app:rewrite_detection}

\paragraph{Problem.}
During counterfactual compare-and-correct generation, the model may not locally repair the original incorrect trajectory. Instead, it may ignore the original reasoning entirely and generate a new solution from scratch. This ``full rewrite'' behavior can lead to reward hacking: when the corrected output happens to be correct, the shaping reward is incorrectly attributed to the ``recoverability'' of the original trajectory, even though the original reasoning process was not actually used.

\paragraph{Detection mechanism.}
We use Python's token-level edit distance to detect full rewrites. Let $d(a,b)$ denote the normalized edit distance between sequences $a$ and $b$, with values in $[0,1]$. When correction succeeds, i.e., $r(x,y)=0 \wedge r(x,\hat y)=1$, we classify the case as a full rewrite and set $\Delta=0$ if both of the following conditions hold:
\begin{equation}
    d(y, \hat y) > \alpha \quad \text{and} \quad d(y, \hat y) > d(\hat y, y^{\mathrm{ref}}),
    \label{eq:rewrite_filter}
\end{equation}
where $\alpha$ is the edit-distance threshold. The first condition detects how far the corrected output deviates from the original trajectory. The second condition confirms that the corrected output is closer to the reference answer than to the original trajectory, meaning that the model tends to copy the reference rather than repair the original reasoning.

Combining this with Eq.~\eqref{eq:app_delta}, the complete shaping term with rewrite filtering is:
\begin{equation}
    \Delta(x;y,y^{\mathrm{ref}})
    =
    \rho\cdot \mathbb{I}\!\left[r(x,y)=0 \wedge r(x,\hat y)=1 \wedge \neg\text{rewrite}(y,\hat y,y^{\mathrm{ref}})\right].
    \label{eq:app_delta_filtered}
\end{equation}

\paragraph{Applicability to code tasks.}
The applicability of IBPO does not depend on whether the output is natural-language reasoning text; it depends on whether the training task provides verifiable terminal feedback. Code-generation tasks naturally satisfy this condition: candidate programs can receive binary or scalar rewards through unit tests, execution results, or compilation feedback. Therefore, IBPO can directly treat multiple candidate programs for the same programming problem as counterfactual trajectories and use differences between programs that pass and fail tests to construct process-sensitive shaping signals.

In our LiveCodeBench experiments, program semantic correctness is always determined by execution-based test feedback, not by edit distance. Edit distance is used only as a structural-fidelity check in the current compare-and-correct instantiation: it determines whether the corrected program still preserves much of the local structure of the original failing program, thereby avoiding the erroneous attribution of a fully rewritten correct program to the ``recoverability'' of the original failed trajectory. Because program semantics are often expressed through local token dependencies, syntactic structure, and control-flow organization, higher token/structural overlap in practice means that the correction process is more likely repairing the original trajectory rather than generating an unrelated solution from scratch.

At the same time, the IBPO framework itself is not tied to token-level edit distance. For code tasks, the comparison operator $\mathcal{M}$ can naturally be replaced or extended with code-specific signals such as AST distance, control-flow-graph similarity, execution-trace similarity, failing-test localization, or compiler-error localization. We use edit distance in this paper to keep the instantiation simple, task-agnostic, and computationally inexpensive; code-specific comparison operators are a direct extension.

\paragraph{Threshold sensitivity analysis.}
We conduct a sensitivity analysis over the threshold $\alpha$ on Qwen3-32B (AIME25). Results are shown in Table~\ref{tab:rewrite_threshold}.

\begin{table}[h]
\centering
\small
\begin{tabular}{ccc}
\toprule
\textbf{Threshold $\alpha$} & \textbf{Flagged as Rewrite (\%)} & \textbf{AIME25 (\%)} \\
\midrule
50\% & 27.45 & 85.3 \\
55\% & 23.36 & 85.5 \\
60\% & 18.42 & 85.6 \\
65\% & 13.37 & 85.5 \\
70\% & 9.82 & 85.4 \\
75\% & 5.73 & 85.3 \\
80\% & 2.06 & 85.3 \\
\bottomrule
\end{tabular}
\caption{Rewrite-detection rate and AIME25 accuracy under different edit-distance thresholds $\alpha$ (Qwen3-32B). Performance remains stable in the 55\%--70\% range, and 60\% is selected as the midpoint of this plateau.}
\label{tab:rewrite_threshold}
\end{table}

\paragraph{Adaptive threshold.}
A fixed threshold is not ideal for cross-domain generalization. A more robust alternative is distribution-based anomaly detection: compute the mean $\mu_d$ and standard deviation $\sigma_d$ of all edit distances in the current batch, and set the threshold to $\alpha=\mu_d+2\sigma_d$. Chebyshev's inequality guarantees that, even for non-normal distributions, at least 75\% of the data lie within $\mu_d \pm 2\sigma_d$, making this criterion conservative and valid for arbitrary distributions.

\paragraph{Suppressing rewrites at the source through reinforcement learning.}
Beyond post-hoc detection, we also incorporate an edit-distance constraint directly into the reward function: when the edit distance between the corrected output and the original incorrect trajectory is too large, we apply an additional negative reward penalty. Concretely, multiple corrected outputs are generated simultaneously, and contrastive reinforcement learning is applied within the group so that the model naturally learns during training that local repair receives higher return than full rewriting. This suppresses the tendency to rewrite at the behavior-policy level. It is a more fundamental solution than threshold filtering: instead of discarding samples after rewriting occurs, it uses incentives to make the model actively avoid rewriting.

\paragraph{Positioning.}
The edit-distance check is a defensive safeguard against reward hacking, not a core component of the IBPO framework. When the model has sufficient base capability, full rewrites are intrinsically rare events, and the specific threshold choice has negligible influence on the main experimental results. Combined with the RL training mechanism above, the model's rewrite tendency further decreases as training proceeds, reducing dependence on the threshold.

\subsection{Trajectory Diversity and Coupling Reduction}
\label{app:decoupling}

IBPO relies on sampling sufficiently diverse trajectories for the same input. In our experiments, we use the following strategies to increase trajectory diversity and reduce trajectory coupling:

\begin{itemize}
    \item \textbf{Stochastic decoding.} We use different random seeds, temperatures, and top-$p$/top-$k$ sampling parameters.
    \item \textbf{Prompt perturbation (optional).} We apply slight perturbations to the system prompt or formatting prompt to induce trajectory-level differences.
\end{itemize}

\section{Implementation Details}
\label{app:impl}

\paragraph{Infrastructure.}
Experiments are conducted on 32 Nvidia A800 (80G) GPUs.

\paragraph{Optimization.}
We use an initial learning rate of $5\times10^{-7}$ with cosine decay, a minimum ratio of 0.1, and linear warmup for 3\% of the total steps. The entropy-regularization coefficient is set to $0$.

\paragraph{Sampling.}
GSPO uses $G=64$ rollouts per prompt, while IBPO uses $G=8$ rollouts to approximately match the overall compute budgets of the methods.

\section{One Instantiation of IBPO: Integration with GSPO}
\label{sec:ibpospo1}

GSPO is used only as a carrier optimizer; replacing GSPO with GRPO or PPO does not change the IBPO formulation. In the preceding sections, we introduced the general IBPO formulation and its reward-shaping definition. IBPO is a training formulation orthogonal to the underlying sequence-level reinforcement-learning method. This section describes how to integrate this formulation seamlessly into GSPO, a representative sequence-level RL algorithm.

\subsection{Preliminaries: Sequence-Level Reinforcement Learning}

We view the autoregressive language model parameterized by $\theta$ as a policy $\pi_\theta$. Let $\mathcal{D}$ denote the query set. Given a query $x$, the model generates a full response $y=(y_1,\dots,y_{|y|})$, with sequence probability
\begin{equation}
    \pi_\theta(y\mid x)=\prod_{t=1}^{|y|}\pi_\theta(y_t\mid x,y_{<t}).
\end{equation}

We consider a general class of sequence-level policy-optimization objectives:
\begin{equation}
    J(\theta)
    =
    \mathbb{E}_{x \sim \mathcal{D},\, y \sim \pi_{\theta_{\mathrm{old}}}(\cdot \mid x)}
    \Bigl[
    \mathcal{L}\bigl(
    s(\theta; x, y),
    \hat{A}(x,y)
    \bigr)
    \Bigr],
    \label{eq:general_seq_rl}
\end{equation}
where $s(\theta; x, y)$ denotes the sequence-level importance-sampling weight, and $\hat A(x,y)$ is constructed from sequence-level rewards. Methods such as GSPO and GRPO can be viewed as specific instantiations of this formulation.

IBPO does not change the optimization form in Eq.~\eqref{eq:general_seq_rl}. Instead, it shapes the original sequence-level reward through the model's self-correction process.

\subsection{IBPO and GSPO: Single-Pass Joint Training}
\label{sec:ibpo_gspo}

At each iteration, we perform a \textbf{single} policy update. For the same batch of queries $x$, we simultaneously construct the candidate response set for sequence-level RL and the self-correction results used to evaluate recoverability. IBPO modifies only the reward definition, shaping it into $r'$, while keeping the GSPO surrogate objective unchanged.

\paragraph{(A) GSPO backbone with shaped rewards.}
For each query $x$, sample $G$ responses $\{y_i\}_{i=1}^{G}$ from the old policy. The sequence-level importance ratio in GSPO is defined as
\begin{equation}
    s_i(\theta)=
    \left(
    \frac{\pi_\theta(y_i\mid x)}
    {\pi_{\theta_{\text{old}}}(y_i\mid x)}
    \right)^{\!\frac{1}{|y_i|}}.
\end{equation}
We use the \textbf{shaped reward} to construct within-group advantages:
\begin{equation}
    \hat A_i=
    \frac{r'(x,y_i)-\mu'}{\sigma'},
\end{equation}
where $\mu'$ and $\sigma'$ denote the mean and standard deviation of the shaped rewards within the group. The GSPO objective is
\begin{equation}
    \begin{aligned}
        J_{\mathrm{GSPO}}(\theta)
        =
        \mathbb{E}_{x \sim \mathcal{D},\, \{y_i\}_{i=1}^{G} \sim \pi_{\theta_{\mathrm{old}}}(\cdot \mid x)}
        \Biggl[
        \frac{1}{G}
        \sum_{i=1}^{G}
        \min\Bigl(
        & s_i(\theta)\,\hat{A}_i, \\
        \operatorname{clip}\!\bigl(s_i(\theta),\,1-\epsilon,\,1+\epsilon\bigr)\,\hat{A}_i
        \Bigr)
        \Biggr].
    \end{aligned}
\end{equation}

\paragraph{Masked GSPO variant.}
For IBPO-mask, we keep the sequence-level advantage $\hat A_i$ unchanged and restrict GSPO's sequence-level log-ratio to tokens selected by the mask. Let $\mathbf{m}_i=(m_{i,1},\ldots,m_{i,|y_i|})$ and $M_i=\max(1,\sum_t m_{i,t})$. Define
\begin{equation}
    s_i^{\mathrm{mask}}(\theta)
    =
    \exp\left(
    \frac{1}{M_i}
    \sum_{t=1}^{|y_i|}
    m_{i,t}
    \left[
    \log \pi_\theta(y_{i,t}\mid x,y_{i,<t})
    -
    \log \pi_{\theta_{\mathrm{old}}}(y_{i,t}\mid x,y_{i,<t})
    \right]
    \right).
    \label{eq:gspo_mask_ratio}
\end{equation}
Then replace $s_i(\theta)$ in the GSPO backbone objective above with $s_i^{\mathrm{mask}}(\theta)$. The first-order gradient of this form is equivalent to broadcasting the same sequence-level advantage $\hat A_i$ to the selected tokens and multiplying by $m_{i,t}$. It introduces neither a token-level value model nor a token-level advantage estimator.

\paragraph{(B) Self-correction shaping signal.}
To compute the shaped reward $r'(x,y_i)$, we construct a self-correction instance for each incorrect answer $y_i$. Specifically, we randomly sample a reference answer $y_i^{\text{ref}}$ (from correct answers if the group contains at least one correct answer; otherwise from other incorrect answers satisfying $j\ne i$), and ask the model to compare and correct:
\begin{equation}
    \hat y_i=\mathcal{C}(x; y_i, y_i^{\text{ref}}).
\end{equation}
This process introduces no additional supervision and is used only to assess whether the model can correct an incorrect answer into a correct one. Combined with the full-rewrite filtering in Eq.~\eqref{eq:app_delta_filtered}, we define
\begin{equation}
    \Delta_i=
    \rho\cdot
    \mathbb{I}\!\left[
    r(x,y_i)<1\;\wedge\;r(x,\hat y_i)=1
    \;\wedge\;\neg\mathrm{rewrite}(y_i,\hat y_i,y_i^{\mathrm{ref}})
    \right],
    \qquad \rho=0.5,
    \label{eq:ibpo_delta}
\end{equation}
and obtain the sequence-level shaped reward:
\begin{equation}
    r'(x,y_i)
    =
    r(x,y_i)+\lambda\,\Delta_i.
    \label{eq:ibpo_shaped_reward}
\end{equation}
Although $r'(x,y_i)$ remains formally a sequence-level scalar, its value depends on the counterfactual self-correction process and therefore implicitly encodes process-level, or step-level, credit-assignment information.

\paragraph{(C) Joint GSPO training for correction behavior.}

In addition to sequence-level optimization of the original reasoning input $x$ using the implicit process reward $r'(x,y)$, we further treat \textbf{correction behavior itself as a reasoning task for the same policy on an extended input space} and train it jointly with the \textbf{same GSPO objective}. Formally, this process does not introduce a new Markov decision process (MDP); it simply applies the policy to different conditional inputs, i.e., a prompt-conditioned policy.

Concretely, for each answer $y_i$ whose recoverability needs to be evaluated, we construct a correction input:
\begin{equation}
    \tilde{x}_i = (x;\; y_i,\; y_i^{\text{ref}}),
\end{equation}
where $y_i^{\text{ref}}$ denotes the reference answer. Conditioned on this input, the model generates a corrected output and receives a terminal reward according to final-answer correctness. Importantly, correction inputs are usually different: each incorrect answer $y_i$ produces its own $\tilde{x}_i$. Therefore, we \textbf{do not} mix corrected outputs from different $\tilde{x}_i$ into the same GSPO group for advantage normalization; instead, each correction input forms its own group.

Specifically, for each correction input $\tilde{x}_i$ enabled for joint training, we sample $G_c$ corrected outputs from the old policy:
\[
    \{z_{i,\ell}\}_{\ell=1}^{G_c}
    \sim
    \pi_{\theta_{\mathrm{old}}}(\cdot\mid \tilde{x}_i).
\]
In implementation, the $\hat y_i$ generated in stage (B) for computing the shaping reward can be reused as one sample in this correction group. If an additional $G_c-1$ corrected outputs are needed, their generation cost is fully included in the training budget. The correction reward is defined as
\[
    r^{\mathrm{corr}}_{i,\ell}=r(x,z_{i,\ell}),
\]
which checks whether the corrected output gives the correct final answer to the original problem $x$. The within-correction-group advantage is
\[
    \hat A^{\mathrm{corr}}_{i,\ell}
    =
    \frac{
    r^{\mathrm{corr}}_{i,\ell}
    -
    \mu_i^{\mathrm{corr}}
    }{
    \sigma_i^{\mathrm{corr}}+10^{-8}
    },
\]
where $\mu_i^{\mathrm{corr}}$ and $\sigma_i^{\mathrm{corr}}$ are computed only over the $G_c$ outputs for the same correction input $\tilde{x}_i$. The auxiliary GSPO objective for correction behavior is:
\begin{equation}
    \begin{aligned}
        J_{\mathrm{corr}}(\theta)
        =
        \mathbb{E}_{\tilde{x}_i}
        \Biggl[
        \frac{1}{G_c}
        \sum_{\ell=1}^{G_c}
        \min\Bigl(
        & s^{\mathrm{corr}}_{i,\ell}(\theta)\,\hat{A}^{\mathrm{corr}}_{i,\ell}, \\
        & \operatorname{clip}\!\bigl(s^{\mathrm{corr}}_{i,\ell}(\theta),\,1-\epsilon,\,1+\epsilon\bigr)\,
        \hat{A}^{\mathrm{corr}}_{i,\ell}
        \Bigr)
        \Biggr],
    \end{aligned}
    \label{eq:gspo_joint}
\end{equation}
where
\[
    s^{\mathrm{corr}}_{i,\ell}(\theta)
    =
    \left(
    \frac{
    \pi_\theta(z_{i,\ell}\mid \tilde{x}_i)
    }{
    \pi_{\theta_{\mathrm{old}}}(z_{i,\ell}\mid \tilde{x}_i)
    }
    \right)^{1/|z_{i,\ell}|}.
\]
The final optimization objective is
\[
    J_{\mathrm{total}}(\theta)
    =
    J_{\mathrm{main}}(\theta)
    +
    \eta J_{\mathrm{corr}}(\theta),
\]
where $J_{\mathrm{main}}$ is the GSPO objective in part (A) based on the original input $x$ and shaped reward $r'(x,y)$, and $\eta$ is the weight of the auxiliary correction objective.

We emphasize that this joint training process \textbf{does not introduce an additional optimization stage or a different loss function}. Learning correction ability is simply behavioral generalization of the policy under different input conditions, enabling the model to gradually internalize comparison and correction ability during training without requiring additional multi-round calls at inference time.

\subsection{IBPO + GSPO Algorithm Pseudocode}
\label{ssec:mvpo_workflow}

Combining the stages above, the full workflow of IBPO + GSPO is summarized in Algorithm~\ref{alg:ibpo_gspo_main} and Algorithm~\ref{alg:ibpo_gspo_corr}.

\begin{algorithm}[H]
    \caption{IBPO + GSPO Main Training Procedure}
    \label{alg:ibpo_gspo_main}
    \begin{algorithmic}[1]
        \Require Dataset $\mathcal{D}$; policy $\pi_\theta$ and old policy $\pi_{\theta_{\mathrm{old}}}$; group size $G$; clipping parameter $\epsilon$; shaping weight $\lambda$; auxiliary weight $\eta$; scale $\rho$; correction operator $\mathcal{C}$; terminal reward $r(\cdot)\in\{0,1\}$.
        \Ensure Updated parameters $\theta$.

        \For{each iteration}
        \State Sample a mini-batch of prompts $\mathcal{B}=\{x\}$ from $\mathcal{D}$.
        \State Initialize batch objective $\mathcal{J}_{\mathcal{B}}(\theta)\leftarrow0$.

        \For{each prompt $x \in \mathcal{B}$}
        \State Sample $G$ answers $\{y_i\}_{i=1}^{G} \sim \pi_{\theta_{\mathrm{old}}}(\cdot \mid x)$ and compute terminal rewards $r_i\leftarrow r(x,y_i)$.

        \For{$i=1$ to $G$}
        \If{$r_i = 0$}
        \If{there exists $j$ such that $r(x,y_j)=1$}
        \State Sample reference answer $y_i^{\mathrm{ref}}$ uniformly from $\{y_j: r(x,y_j)=1\}$.
        \Else
        \State Sample reference answer $y_i^{\mathrm{ref}}$ uniformly from $\{y_j: r(x,y_j)=0,\ j\ne i\}$.
        \EndIf
        \State Generate corrected result $\hat y_i \leftarrow \mathcal{C}(x; y_i, y_i^{\mathrm{ref}})$.
        \State $\Delta_i \leftarrow \rho\cdot \mathbb{I}\!\left[r(x,\hat y_i)=1 \wedge \neg\mathrm{rewrite}(y_i,\hat y_i,y_i^{\mathrm{ref}})\right]$.
        \Else
        \State $\Delta_i \leftarrow 0$.
        \EndIf
        \State $r'_i \leftarrow r_i + \lambda\,\Delta_i$, and compute sequence-level ratio $s_i(\theta)$.
        \EndFor

        \State $\mu' \leftarrow \frac{1}{G}\sum_{i=1}^{G} r'_i$.
        \State $\sigma' \leftarrow \sqrt{\frac{1}{G}\sum_{i=1}^{G}(r'_i-\mu')^2}$.
        \For{$i=1$ to $G$}
        \State $\hat A_i \leftarrow \frac{r'_i-\mu'}{\sigma' + 10^{-8}}$.
        \EndFor

        \State $\mathcal{J}_{x}(\theta) \leftarrow
        \frac{1}{G}\sum_{i=1}^{G}
        \min\!\Big(
        s_i(\theta)\hat A_i,\;
        \mathrm{clip}(s_i(\theta),1-\epsilon,1+\epsilon)\hat A_i
        \Big)$.
        \State If joint training is enabled, compute $\mathcal{J}^{\mathrm{corr}}_x(\theta)$ using Algorithm~\ref{alg:ibpo_gspo_corr}; otherwise set $\mathcal{J}^{\mathrm{corr}}_x(\theta)\leftarrow0$.
        \State $\mathcal{J}_{\mathcal{B}}(\theta)\leftarrow
        \mathcal{J}_{\mathcal{B}}(\theta)+
        \mathcal{J}_{x}(\theta)+
        \eta\mathcal{J}^{\mathrm{corr}}_x(\theta)$.
        \EndFor
        \State Update $\theta$ by maximizing $\mathcal{J}_{\mathcal{B}}(\theta)$.
        \EndFor
    \end{algorithmic}
\end{algorithm}

\begin{algorithm}[H]
    \caption{Optional GSPO Auxiliary Training for Correction Behavior}
    \label{alg:ibpo_gspo_corr}
    \begin{algorithmic}[1]
        \Require Original prompt $x$; incorrect-answer set $\{y_i:r_i=0\}$; reference answer $y_i^{\mathrm{ref}}$; generated correction $\hat y_i$; correction group size $G_c$; clipping parameter $\epsilon$; terminal reward $r(\cdot)$.
        \Ensure Auxiliary correction objective $\mathcal{J}^{\mathrm{corr}}_x(\theta)$.

        \State Initialize $\mathcal{J}^{\mathrm{corr}}_x(\theta)\leftarrow0$.
        \For{each incorrect answer $y_i$ enabled for joint training}
        \State Construct correction input $\tilde{x}_i \leftarrow (x;\,y_i,\,y_i^{\mathrm{ref}})$.
        \State Reuse $\hat y_i$ as $z_{i,1}$; if $G_c>1$, sample additional outputs $\{z_{i,\ell}\}_{\ell=2}^{G_c}\sim\pi_{\theta_{\mathrm{old}}}(\cdot\mid\tilde{x}_i)$.
        \State Count all additional correction sampling, verification, and comparison costs in the training budget.
        \For{$\ell=1$ to $G_c$}
        \State Compute $r^{\mathrm{corr}}_{i,\ell}\leftarrow r(x,z_{i,\ell})$ and sequence-level ratio $s^{\mathrm{corr}}_{i,\ell}(\theta)$.
        \EndFor
        \State Compute $\mu_i^{\mathrm{corr}}$, $\sigma_i^{\mathrm{corr}}$, and $\hat A^{\mathrm{corr}}_{i,\ell}$ within the $G_c$ outputs for the same $\tilde{x}_i$.
        \State Accumulate $\mathcal{J}^{\mathrm{corr}}_i(\theta)$ according to Eq.~\eqref{eq:gspo_joint}; different $\tilde{x}_i$ do not share within-group normalization statistics.
        \State $\mathcal{J}^{\mathrm{corr}}_x(\theta)\leftarrow
        \mathcal{J}^{\mathrm{corr}}_x(\theta)+\mathcal{J}^{\mathrm{corr}}_i(\theta)$.
        \EndFor
    \end{algorithmic}
\end{algorithm}


\newpage
\section*{NeurIPS Paper Checklist}

\begin{enumerate}

\item {\bf Claims}
    \item[] Question: Do the main claims made in the abstract and introduction accurately reflect the paper's contributions and scope?
    \item[] Answer: \answerYes{} 
    \item[] Justification: The abstract and introduction state the IBPO framework, its scope, and its empirical focus on math and code reasoning. The claims are further qualified in the method, theory, and limitations sections.

\item {\bf Limitations}
    \item[] Question: Does the paper discuss the limitations of the work performed by the authors?
    \item[] Answer: \answerYes{} 
    \item[] Justification: The paper includes a dedicated limitations section discussing compute overhead, dependence on reliable terminal verification, verifier noise, and possible mitigations.

\item {\bf Theory assumptions and proofs}
    \item[] Question: For each theoretical result, does the paper provide the full set of assumptions and a complete (and correct) proof?
    \item[] Answer: \answerYes{} 
    \item[] Justification: The theoretical analysis states the exchangeability, effective negative correlation, and non-degeneracy conditions, and provides a finite-sample proof for the pre-normalization centered advantage term in the appendix.

    \item {\bf Experimental result reproducibility}
    \item[] Question: Does the paper fully disclose all the information needed to reproduce the main experimental results of the paper to the extent that it affects the main claims and/or conclusions of the paper (regardless of whether the code and data are provided or not)?
    \item[] Answer: \answerYes{} 
    \item[] Justification: The paper reports datasets, base models, context lengths, decoding settings, rollout counts, optimization hyperparameters, evaluation protocol, baselines, and compute matching details in the main text and appendix. The authors plan to release the code repository after acceptance.

\item {\bf Open access to data and code}
    \item[] Question: Does the paper provide open access to the data and code, with sufficient instructions to faithfully reproduce the main experimental results, as described in supplemental material?
    \item[] Answer: \answerNo{} 
    \item[] Justification: The current submission describes the algorithm and experimental settings, and states that the authors plan to release code after acceptance.

\item {\bf Experimental setting/details}
    \item[] Question: Does the paper specify all the training and test details (e.g., data splits, hyperparameters, how they were chosen, type of optimizer) necessary to understand the results?
    \item[] Answer: \answerYes{} 
    \item[] Justification: The experiments section and appendix specify datasets, models, context lengths, rollout counts, learning rate, schedule, warmup, entropy regularization, batch size, decoding parameters, verifier use, and compute matching.

\item {\bf Experiment statistical significance}
    \item[] Question: Does the paper report error bars suitably and correctly defined or other appropriate information about the statistical significance of the experiments?
    \item[] Answer: \answerYes{} 
    \item[] Justification: Main results report means over five random seeds with 95\% bootstrap confidence intervals and a normal-approximation significance statement for the main method. The evaluation protocol states that 64 independent pass@1 evaluations are averaged.

\item {\bf Experiments compute resources}
    \item[] Question: For each experiment, does the paper provide sufficient information on the computer resources (type of compute workers, memory, time of execution) needed to reproduce the experiments?
    \item[] Answer: \answerYes{} 
    \item[] Justification: The experiments section reports the use of 32 Nvidia A800 80G GPUs and describes the training-compute matching procedure; the appendix explains the generation, correction, verification, filtering, and update costs included in the compute budget.
    
\item {\bf Code of ethics}
    \item[] Question: Does the research conducted in the paper conform, in every respect, with the NeurIPS Code of Ethics \url{https://neurips.cc/public/EthicsGuidelines}?
    \item[] Answer: \answerYes{} 
    \item[] Justification: The research uses public reasoning benchmarks and standard model fine-tuning/evaluation procedures, preserves anonymity, and does not involve human subjects, privacy-sensitive data, or the release of a new high-risk dataset.

\item {\bf Broader impacts}
    \item[] Question: Does the paper discuss both potential positive societal impacts and negative societal impacts of the work performed?
    \item[] Answer: \answerYes{} 
    \item[] Justification: The limitations and societal impact sections discuss potential positive and negative impacts. Positive impacts include improving the compute efficiency of reasoning reinforcement learning and achieving stronger performance.
    
\item {\bf Safeguards}
    \item[] Question: Does the paper describe safeguards that have been put in place for responsible release of data or models that have a high risk for misuse (e.g., pre-trained language models, image generators, or scraped datasets)?
    \item[] Answer: \answerNA{} 
    \item[] Justification: The paper does not release a new high-misuse-risk pretrained model, image generator, or scraped dataset; the work focuses on a training method and evaluates it on existing benchmarks.

\item {\bf Licenses for existing assets}
    \item[] Question: Are the creators or original owners of assets (e.g., code, data, models), used in the paper, properly credited and are the license and terms of use explicitly mentioned and properly respected?
    \item[] Answer: \answerYes{} 
    \item[] Justification: Appendix~\ref{sec:appendix-assets} lists official license information for public assets including Qwen, MathArena, LiveCodeBench, and vLLM, and cites the corresponding models, datasets, and evaluation tools.

\item {\bf New assets}
    \item[] Question: Are new assets introduced in the paper well documented and is the documentation provided alongside the assets?
    \item[] Answer: \answerNA{} 
    \item[] Justification: The paper does not introduce a new dataset or benchmark. Any future released code/checkpoints will be documented separately with release instructions.

\item {\bf Crowdsourcing and research with human subjects}
    \item[] Question: For crowdsourcing experiments and research with human subjects, does the paper include the full text of instructions given to participants and screenshots, if applicable, as well as details about compensation (if any)? 
    \item[] Answer: \answerNA{} 
    \item[] Justification: The work does not involve crowdsourcing experiments or human-subject studies.

\item {\bf Institutional review board (IRB) approvals or equivalent for research with human subjects}
    \item[] Question: Does the paper describe potential risks incurred by study participants, whether such risks were disclosed to the subjects, and whether Institutional Review Board (IRB) approvals (or an equivalent approval/review based on the requirements of your country or institution) were obtained?
    \item[] Answer: \answerNA{} 
    \item[] Justification: The work does not involve crowdsourcing or human-subject research, so IRB approval is not applicable.

\item {\bf Declaration of LLM usage}
    \item[] Question: Does the paper describe the usage of LLMs if it is an important, original, or non-standard component of the core methods in this research? Note that if the LLM is used only for writing, editing, or formatting purposes and does \emph{not} impact the core methodology, scientific rigor, or originality of the research, declaration is not required.
    \item[] Answer: \answerYes{} 
    \item[] Justification: LLMs are central to the method and experiments. The paper describes the base models, LLM rollout generation, correction generation, and verifier-based training/evaluation setup.

\end{enumerate}

\end{document}